\documentclass{article}

\usepackage{iclr2027_conference,times}

\usepackage[utf8]{inputenc}
\usepackage[T1]{fontenc}
\usepackage{hyperref}
\hypersetup{hidelinks}
\usepackage{url}
\usepackage{booktabs}
\usepackage{array}
\usepackage{amsfonts}
\usepackage{xcolor}
\usepackage{amsmath}
\usepackage{amssymb}
\usepackage{amsthm}
\usepackage{graphicx}
\usepackage{wrapfig}
\usepackage{caption}
\usepackage{multirow}
\usepackage{microtype}
\usepackage{algpseudocode}
\usepackage{float}

\floatstyle{ruled}
\newfloat{algorithm}{tbp}{loa}
\floatname{algorithm}{Algorithm}

\newtheorem{theorem}{Theorem}
\newtheorem{lemma}[theorem]{Lemma}
\newtheorem{proposition}[theorem]{Proposition}
\newtheorem{corollary}[theorem]{Corollary}
\newtheorem{definition}[theorem]{Definition}
\newtheorem{assumption}{Assumption}
\newtheorem{remark}{Remark}

\usepackage{pifont}
\newcommand{\xmark}{\ding{55}}
\newcommand{\cmark}{\ding{51}}

\usepackage[capitalise,nameinlink]{cleveref}
\crefname{algorithm}{Algorithm}{Algorithms}
\Crefname{algorithm}{Algorithm}{Algorithms}
\crefname{assumption}{Assumption}{Assumptions}
\Crefname{assumption}{Assumption}{Assumptions}
\crefname{remark}{Remark}{Remarks}
\Crefname{remark}{Remark}{Remarks}

\newcommand{\Aref}[1]{Appendix~\ref{#1}}

\title{Pick the Mixture, Not the Samples:\\Combinatorial Data Selection for LLMs}

\author{
  Luyang Zhang$^{1}$ \quad Jialu Wang$^{2}$ \quad Beibei Li$^{1}$\\
  {\normalfont $^{1}$Carnegie Mellon University}\\
  {\normalfont $^{2}$University of California, Santa Cruz}\\
  {\normalfont\small \href{mailto:luyangz@cmu.edu}{\texttt{luyangz@cmu.edu}},
  \href{mailto:faldict@ucsc.edu}{\texttt{faldict@ucsc.edu}},
  \href{mailto:beibeili@andrew.cmu.edu}{\texttt{beibeili@andrew.cmu.edu}}}
}

\iclrfinalcopy
\usepackage{etoolbox}
\makeatletter
\patchcmd{\@maketitle}{Published as a conference paper at ICLR 2027}{Preprint}{}{\PackageError{arxiv}{Preprint header patch failed}{}}
\makeatother

\begin{document}

\maketitle

\begin{abstract}
Data selection for LLM fine-tuning faces two facts about real training: data pools draw from many distinct domains, and evaluation spans diverse tasks. As a result, the same sample helps different tasks differently, and samples interact during training so that a subset's value is not the sum of its parts. Current data-selection methods leave these realities largely unaddressed. We formalize them as \emph{Combinatorial Data Selection} (CDS), a framework that separates \emph{attribution} (how samples are scored) from \emph{selection} (which subset to pick) and models the non-additive interactions among samples. On a real instruction-tuning data pool, we find substantial cross-domain complementarity (joint utility exceeds the sum of singletons), alongside redundancy within domains. Inspired by this, we propose \emph{Mixture-then-Select}, a two-stage algorithm that first decides how many samples to take from each domain, then ranks samples within each domain by attribution score. Theoretically, the gap to the optimal subset is linear in the selection size and independent of data pool size. Empirically, Mixture-then-Select {improves held-out utility over the evaluated baselines}, is robust across attribution methods, and works across different models and data pool sizes. The framework also extends to the data purchasing problem under a budget constraint when sample prices vary. Our findings suggest that allocating selection across domains offers a promising path toward more efficient and effective data selection for modern LLMs. 
\end{abstract}

\section{Introduction}
\label{sec:intro}
From a large data pool spanning many skill domains (e.g., math, code, dialog, instruction following), which small subset should we train on for LLM fine-tuning? Training on a well-chosen small subset often matches or beats training on the full data pool~\citep{zhou2023lima,liu2024deita,xia2024less,li2024ifd}: what is selected matters more than how many. But two features of training data make selection hard. First, samples help different evaluation tasks unevenly: a sample that helps one task may be useless or harmful for another. Second, samples interact during training: training on two samples together can yield more than the sum of their contributions (\emph{complementary}) or less (\emph{substitutive}).

Together, these facts mean per-sample scoring is fundamentally insufficient: it ignores how chosen samples combine and which tasks each is meant to serve. We formalize this as \emph{Combinatorial Data Selection} (CDS), in which sample scoring (\emph{attribution}) and subset choice (\emph{selection}) are separate decisions. CDS extends classical submodular selection~\citep[e.g.,][]{wei2015submodularity,mirzasoleiman2020coresets} in two ways: (i) it admits both complementary and substitutive sample interactions, whereas submodular formulations capture only the substitutive (diminishing-returns) case; and (ii) it handles multiple downstream tasks jointly rather than one at a time. Any per-sample attribution score~\citep[e.g.,][]{xia2024less,koh2017understanding} plugs into the rule.

As solving CDS exactly is intractable, we propose \emph{Mixture-then-Select} (MtS), a two-stage algorithm that first decides how many samples to take from each domain based on per-domain utility measurements (capturing cross-domain effects), then selects the top samples within each domain by attribution score. MtS is attribution-agnostic (it plugs into any per-sample score, whether oracle values, gradient-based, loss-based, or other) with a theoretical guarantee: its utility is provably within an additive gap of the optimal subset's that grows linearly in the number of samples picked and is independent of data pool size, so quality does not degrade as pools grow.

Empirically, we evaluate MtS across diverse skill domains, evaluation tasks, model families (Qwen3-8B, Llama-3.1-8B-Instruct), and data pool sizes (50K to 200K). We first find that the data pool exhibits substantial cross-domain complementarity, which MtS's domain-mixture decision exploits. Results show that MtS outperforms eight baselines across four selection principles {in held-out utility}: per-sample quality (LESS~\citep{xia2024less}, EL2N~\citep{paul2021deep}, Superfiltering~\citep{li2024superfiltering}, DEITA~\citep{liu2024deita}), diversity ($k$-DPP~\citep{kulesza2011kdpp}, Facility Location~\citep{wei2015submodularity}), gradient matching (CRAIG~\citep{mirzasoleiman2020coresets}), and partition-based (Stratified). {Additional comparisons with recent mixture-optimization methods across diverse downstream benchmarks provide further evidence of MtS's robustness across evaluation tasks and data domain partitions.} MtS also applies to buying training data when samples have different prices and the buyer has a fixed budget. Across every budget level we tested, MtS reaches $\geq 89\%$ of the best achievable performance and matches the optimum on $58\%$ instances.

In summary, this paper makes three contributions. \textbf{(i) Framework.} We frame LLM data selection as combinatorial allocation across domains, separating attribution from selection and capturing both complementary and substitutive sample interactions. \textbf{(ii) Methodology.} We propose a two-stage algorithm that first picks the per-domain mixture and then ranks samples within each domain, with both theoretical and empirical performance guarantees. \textbf{(iii) Application.} MtS gives LLM practitioners a practical algorithm for a real-world problem: deciding which training data to buy under a fixed budget.

\section{Related work}
\label{sec:related}

\begin{table}[!h]
\vspace{-5pt}
\centering
\footnotesize
\setlength{\tabcolsep}{3pt}
\renewcommand{\arraystretch}{0.95}
\caption{Comparison with related methods. Columns indicate whether a method models sample-pair interactions, supports cardinality or nonuniform-price budgets, provides a subset-utility approximation guarantee, and optimizes domain allocation.\protect\footnotemark{}}
\label{tab:comparison}
\resizebox{\linewidth}{!}{%
\begin{tabular}{@{}llccccc@{}}
\toprule
\textbf{Direction} & \textbf{Method} & \textbf{Interactions} & \textbf{Size} & \textbf{Budget} & \textbf{Guarantee} & \textbf{Mixture} \\
\midrule
\multirow{2}{*}{Data attribution} & Data Shapley~\citep{ghorbani2019data}                                                & \xmark & \xmark & \xmark & \xmark & \xmark \\
                                  & Per-sample scores                                                                    & \xmark & \xmark & \xmark & \xmark & \xmark \\
\midrule
\multirow{6}{*}{Data selection and distillation} & Dataset distillation~\citep{wang2018dataset,zhao2021dataset}        & \xmark & \cmark & \xmark & \xmark & \xmark \\
                                                 & $k$-DPP~\citep{kulesza2011kdpp}                                     & \xmark & \cmark & \xmark & \xmark & \xmark \\
                                                 & DoReMi~\citep{xie2023doremi} / RegMix~\citep{liu2025regmix}         & \xmark & \xmark & \xmark & \xmark & \cmark \\
                                                 & Stratified                                                          & \xmark & \cmark & \xmark & \xmark & \cmark \\
                                                 & CRAIG~\citep{mirzasoleiman2020coresets} / Facility Loc.~\citep{wei2015submodularity} & \xmark & \cmark & \xmark & \cmark & \xmark \\
                                                 & SMART~\citep{smart2024}                                             & \xmark & \cmark & \xmark & \cmark & \cmark \\
\midrule
\multirow{2}{*}{Data pricing and markets} & Data Banzhaf~\citep{wang2023banzhaf} / DVRL~\citep{yoon2020data}                          & \xmark & \xmark & \xmark & \xmark & \xmark \\
                                          & Fairshare~\citep{zhang2025fairshare} / Data Market Design~\citep{ravindranath2023datamarket} & \xmark & \xmark & \cmark & \xmark & \xmark \\
\midrule
Ours & Mixture-then-Select                                                                                                                  & \cmark & \cmark & \cmark & \cmark & \cmark \\
\bottomrule
\end{tabular}%
}
\vspace{-8pt}
\end{table}
\footnotetext{Similarity kernels and gradient surrogates do not count as Interactions; convergence guarantees on inner optimization do not count as Guarantee.}

\textbf{Data attribution.}
Per-sample scores fall on a cost spectrum: a single forward pass for instruction-tuning quality classifiers~\citep{liu2024deita,li2024superfiltering,li2024ifd} and EL2N~\citep{paul2021deep}; per-sample gradients for LESS~\citep{xia2024less}, TracIn~\citep{pruthi2020tracin}, and Hessian-based influence approximations~\citep{kwon2023datainf,basu2020influence,koh2017understanding}; repeated training-trace queries for Shapley~\citep{ghorbani2019data} and datamodels~\citep{ilyas2022datamodels}. Downstream selection ranks and thresholds the scores.

\textbf{Data selection and distillation.}
Subset-level methods reason about $S$ as a whole: submodular and coreset approaches under cardinality constraints~\citep{wei2015submodularity,mirzasoleiman2020coresets,smart2024,maharana2023d2,nemhauser1978analysis,bian2017guarantees,das2011submodular,harshaw2019submodular,sviridenko2004note,mirzasoleiman2015lazier,feige1998threshold,vondrak2013symmetry} (CRAIG, Facility Location), determinantal point processes~\citep{kulesza2011kdpp,kothawade2022prism}, and dataset distillation~\citep{wang2018dataset,zhao2021dataset}. {DoReMi~\citep{xie2023doremi} learns proxy-derived domain proportions; RegMix~\citep{liu2025regmix} predicts them from proxy-mixture runs; DIDS~\citep{shi2025dids} updates sampling ratios during training. MtS instead studies a complementary question: which individual samples to acquire before training under a cardinality or price budget, accounting for pairwise interactions.} M\"obius interaction indices~\citep{grabisch1999axiomatic,tsai2022faithshap} are the decomposition tool our framework uses for pairwise utility.

\textbf{Data pricing and markets.}
Data pricing values data as an economic asset: Shapley~\citep{ghorbani2019data} and Banzhaf~\citep{wang2023banzhaf} contributions, RL-based valuation~\citep{yoon2020data}, market design under incentive constraints~\citep{ravindranath2023datamarket}, LLM-specific fair-share pricing~\citep{zhang2025fairshare}, and attribution-aware creator compensation in generative music~\citep{zhang2026s}. \Cref{tab:comparison} summarizes the three threads.

\section{Combinatorial data selection}
\label{sec:formulation}
\label{sec:algorithms}
\label{sec:theory}

Training-data selection is a combinatorial problem: a subset's value depends on how its samples interact during fine-tuning, not just on which samples it contains. We formalize selection as a constrained optimization (\Cref{sec:problem-setup}), propose an efficient algorithm to solve it (\Cref{sec:mg_algorithm}), and prove a theoretical performance guarantee (\Cref{sec:spec_main}).

\begin{figure}[t]
\centering
\includegraphics[width=\linewidth]{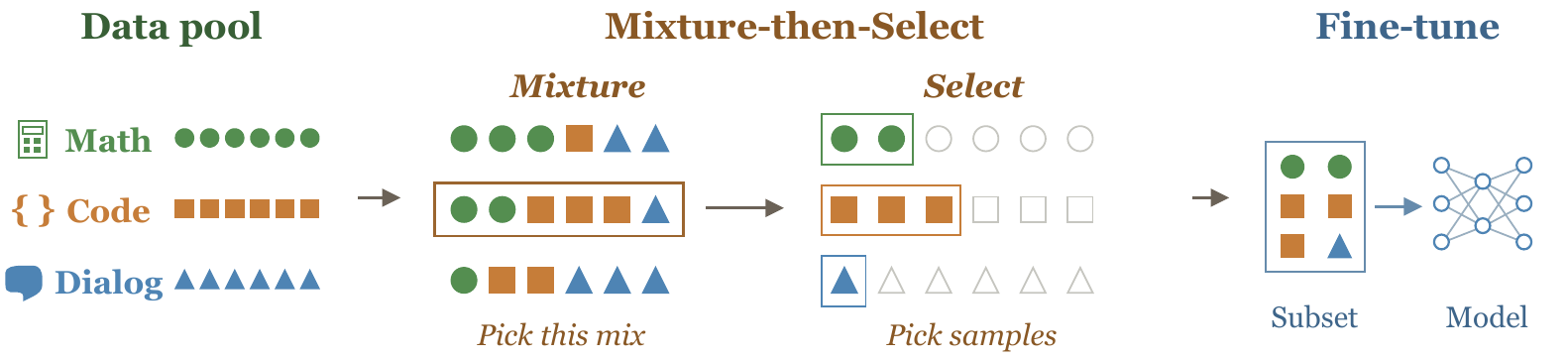}
\caption{Mixture-then-Select first uses a learned utility predictor to choose how many samples to take from each domain, then selects samples ranked by attribution score within each domain. The selected subset is used for fine-tuning.}
\label{fig:framework}
\vspace{-5pt}
\end{figure}

\subsection{Data selection and pairwise interactions}
\label{sec:problem-setup}

We introduce the data selection problem, modeling how training samples interact during fine-tuning, and a structural decomposition of those interactions.

\textbf{The selection problem.}
An LLM developer starts with a base model and a data pool $\mathcal{D} = \{1, \ldots, N\}$ of $N$ candidate training samples, each carrying a per-sample cost $p_j > 0$. The LLM developer cares about performance across $T$ evaluation tasks (e.g., code generation, mathematical reasoning, instruction following). The data pool is partitioned by skill into $T$ \emph{domains} $\mathcal{D} = \bigsqcup_{c=1}^T C_c$, one per task. For each task $t$, fine-tuning on any subset $S \subseteq \mathcal{D}$ yields a task utility $u_t(S) \geq 0$, the improvement in task-$t$ performance \emph{over the base model}, with $u_t(\emptyset) = 0$ by construction and higher $u_t$ corresponding to a better-trained model. These per-task utilities combine into a single scalar objective via an aggregation function $f$:
\begin{equation}
    u(S) = f\bigl(u_1(S), \ldots, u_T(S)\bigr).
    \label{eq:multitask}
\end{equation}
Modern LLMs are evaluated across multiple tasks, which developers may prioritize differently. A natural choice for $f$ to incorporate this multi-objective structure is the \emph{weighted average} $f = \sum_t w_t u_t$, with task weight $w_t$ denoting the relative importance the LLM developer attaches to task $t$ (e.g., math weighted above chit-chat); alternative aggregators (minimum, concave) are discussed in \Aref{app:theory}. 

With a budget $B$ on total cost, the LLM developer selects a subset $S \subseteq \mathcal{D}$ to maximize the aggregate utility $u(S)$ subject to $c(S) = \sum_{j\in S} p_j \le B$. Prior data-selection methods~\citep[e.g.,][]{xia2024less} are \textbf{special cases} of this formulation in two ways: they target a single task ($T{=}1$) and assume uniform per-sample cost ($p_j{=}1$, a cardinality constraint). Our framework is \textbf{more general}, supporting arbitrary $T$, any aggregator $f$, and heterogeneous per-sample costs, making it a natural fit for data-marketplace settings where different vendors set different prices~\citep[e.g.,][]{zhang2025fairshare}. 

\textbf{Pairwise approximation of set utility.}
Training samples interact during fine-tuning, so the set utility $u(S)$ is non-additive: the utility of a subset is not the sum of its singleton utilities~\citep{ghorbani2019data}. The full set function thus has $2^N$ possible values and cannot be reduced to per-sample terms. To capture the interactions tractably, we use the M\"obius inversion formula~\citep{grabisch1999axiomatic, tsai2022faithshap}, a decomposition of set utility into individual-sample contributions and interactions of increasing order. Keeping only the singleton and pairwise contributions yields
\begin{equation}
    u(S) = \sum_{j \in S} v_j + \sum_{\{i,j\} \subseteq S} \varphi_{ij} + \varepsilon(S),
    \label{eq:decomposition}
\end{equation}
where $v_j = u(\{j\})$ is the singleton value, $\varphi_{ij} = u(\{i,j\}) - v_i - v_j$ is the pairwise interaction, and $\varepsilon(S)$ collects contributions from triplets and higher-order interactions. Pairwise truncation is the standard choice for such decompositions~\citep{grabisch1999axiomatic, tsai2022faithshap}; on our data pool, the pairwise approximation $\tilde u$ captures $\geq 90\%$ of the variance in $u(S)$ across $\sim 100$ measured triplets ($R^2 = 0.90$, with pairwise residuals below the seed-noise floor; \Aref{app:theory}), and higher-order interactions are infeasible to estimate at our data pool size. The pairwise interactions form a symmetric matrix $\Phi := (\varphi_{ij})_{i,j}$. We write $\tilde u(S) := u(S) - \varepsilon(S) = \sum_{j \in S} v_j + \sum_{\{i,j\} \subseteq S} \varphi_{ij}$ for this pairwise approximation, the deterministic truncation our analysis works with.

\textbf{Decomposition of the pairwise utility.}
The pairwise utility matrix $\Phi$ encodes interactions between every pair of samples, but a selection algorithm only faces two coarser choices: \emph{which domains $S$ draws from} (the mixture composition $\mathbf{n}(S) = (|S \cap C_c|)_{c=1}^T$) and \emph{which specific samples within each domain}. To align the matrix with these two choices, we decompose $\Phi$ along the domain partition $\mathcal{D} = \bigsqcup_c C_c$ into three pieces:
\begin{equation}
    \Phi \;=\; \Phi_w \;+\; \Phi_{\text{block}} \;+\; \Phi_{\text{res}},
    \label{eq:decomp}
\end{equation}
where the within-domain interaction matrix $\Phi_w$ is block-diagonal (non-zero only on within-domain pairs); the cross-domain block-mean matrix $\Phi_{\text{block}}$ is zero within each domain and equals the cross-domain block mean $\bar\varphi_{c,c'} := \frac{1}{|C_c||C_{c'}|}\sum_{i \in C_c, j \in C_{c'}} \varphi_{ij}$ on each cross-domain block ($c \neq c'$); and the residual matrix $\Phi_{\text{res}} = \Phi - \Phi_w - \Phi_{\text{block}}$ captures the deviation from those block means. The within-domain choice is governed by $\Phi_w$, and the mixture choice by $\Phi_{\text{block}}$. {Equation~\ref{eq:decomp} separates within-domain and cross-domain interactions. We evaluate the pairwise approximation empirically in Section~\ref{sec:interaction_validation}.}

\subsection{The Mixture-then-Select algorithm}
\label{sec:mg_algorithm}

Selecting a subset $S$ to maximize Eq.~(\ref{eq:multitask}) is hard for two reasons. First, the search space of subsets is exponential in the data pool size $N$, so direct enumeration is infeasible (the problem is APX-hard, ruling out arbitrarily good polynomial-time approximation; Theorem~\ref{thm:lower_bound}). Second, sample-level methods like LESS~\citep{xia2024less}, DEITA~\citep{liu2024deita}, and influence-based selectors~\citep{koh2017understanding} score each sample independently and take the top-$m$ by score, which ignores how the resulting subset distributes across domains and so cannot exploit the domain mixture. We design an algorithm that exploits the domain structure of the pairwise decomposition Eq.~(\ref{eq:decomp}).

For selection size $m$, any algorithm selecting $S$ from the domain-partitioned data pool makes two distinct choices: the domain mixture (how many samples to take from each domain) and the within-domain identities (which specific samples to take inside each domain). The decomposition $\Phi = \Phi_w + \Phi_{\text{block}} + \Phi_{\text{res}}$ tells us these two choices are decoupled at the pairwise level. The cross-domain block-mean term $\Phi_{\text{block}}$ depends on $S$ only through the mixture $\mathbf{n}(S) = (n_c)_{c=1}^T$, so Stage~1 can target it with a small $T$-dimensional search over feasible mixtures. The within-domain term $\Phi_w$ is fixed once the within-domain identities are, so Stage~2 can target it by ranking samples within each domain on their singleton values $v_j$. The residual $\Phi_{\text{res}}$ is uncontrolled by either stage and contributes to the bound's gap. MtS implements this two-stage rule: pick the domain mixture first, then fill each domain by within-domain ranking.

{This separation captures cross-domain average interactions through domain counts without optimizing every sample pair. Sample-specific interactions remain outside the static ranking rule and enter the approximation bound through the within-domain and cross-domain residual terms.}

\begin{algorithm}[h]
\caption{Mixture-then-Select~{(exact-allocation form)}}
\label{alg:mg}
\begin{algorithmic}[1]
\Require data pool $\mathcal{D}=\bigsqcup_c C_c$, per-sample values $\{v_j\}_{j\in\mathcal{D}}$, selection size $m$, {allocation-utility predictor} $\hat{g}$
\Ensure subset $S \subseteq \mathcal{D}$ with $|S|=m$
\State $\hat{\mathbf{n}} \leftarrow \arg\max_{\mathbf{n}\in\mathbb{Z}_{\geq 0}^T,\,\sum_c n_c=m,\,n_c\leq |C_c|}\; \hat{g}(\mathbf{n})$ \Comment{Stage 1: predict domain mixture}
\State $S \leftarrow \emptyset$
\For{domain $c \in \{1,\ldots,T\}$}
  \State $S_c \leftarrow$ top $\hat{n}_c$ samples in $C_c$ ranked by $v_j$ \Comment{Stage 2: within-domain selection}
  \State $S \leftarrow S \cup S_c$
\EndFor
\State \Return $S$
\end{algorithmic}
\end{algorithm}

{Algorithm~\ref{alg:mg} states the exact allocation objective for a learned predictor $\hat g(\mathbf n)$, distinct from the analytical pairwise utility $\tilde u(S)$.} Under the uniform-error hypothesis in \Aref{app:predictor}, a rank-$r$ regressor needs $L=O(rT/\varepsilon^2)$ probes at fixed confidence (\Cref{cor:margin-empirical}), whereas direct optimization of $\tilde u$ would require estimating all $\binom{T}{2}$ cross-domain block means in $\Phi_{\text{block}}$. {The reported 50K implementation obtains domain counts from D1's greedy construction using a learned set-level regressor $\hat F(S)$, then refills those counts by within-domain ranking (\Aref{app:implementation}).} The per-sample values $v_j$ may come from any source (oracle losses, gradient-based scores, or influence functions), so MtS drops in over LESS, DEITA, and influence-based selectors without modifying the value source. Algorithm~\ref{alg:mg} is the unit-cost specialization, where every sample has equal cost ($p_j{=}1$) so the budget $B$ from \Cref{sec:problem-setup} reduces to the selection size $m$. The domain-uniform-cost extension (\Cref{sec:spec_main}) replaces this constraint with a price-weighted budget without changing the algorithm. {The empirical allocation search evaluates candidate additions by batched regression; Stage~2 sorts within each domain. Neither search step requires additional fine-tuning runs.}

Formal assumptions on $u$ and $\Phi$, the operator-norm refinement that drives the bound, the two-source decomposition of complementarity (Proposition~\ref{prop:complementarity}), and the adaptive IA-Greedy extension at $T{=}1$ are all in \Aref{app:theory}.

\subsection{Approximation guarantees}
\label{sec:spec_main}

\Cref{sec:mg_algorithm} gave an algorithm; we now bound its utility gap to the optimal subset $S^*$ in Eq.~(\ref{eq:multitask}). A worst-case guarantee explains how two-stage decoupling scales with selection size and how errors in the interaction matrix $\Phi$ and singleton values $v_j$ affect selection quality. We use $m=B$ under unit costs and discuss the cost-weighted extension below. Under the stated assumptions, the structural gap grows linearly in $m$, rather than quadratically, and is independent of pool size $N$ (\Cref{thm:spec-main}). We then establish tightness (\Cref{thm:tightness}) and robustness to estimated $\Phi$ and $v_j$, with additional error proportional to estimation error (\Cref{thm:robust-op,cor:vj-robust}).

The bound decomposes the gap by stage: Stage~1 incurs a cross-domain term and Stage~2 incurs a within-domain term, each controlled by the \emph{operator norm} of a residual matrix, a single number summarising the matrix's dominant signal direction (formally, its largest singular value). This is a factor-$m$ improvement over the entrywise Bounded-Complementarity (BC($\beta$)) form~\citep{bian2017guarantees}, which puts a uniform upper bound on every entry of $\Phi$; we compare both forms after Theorem~\ref{thm:spec-main}. Intuitively, the operator norm captures $\Phi$'s coherent signal while BC($\beta$) sums up worst-case per-pair magnitudes that often cancel in low-rank interaction matrices; this is why the spectral form is dramatically tighter on real $\hat\Phi$ ($\approx 51\times$ on our calibrated data pool, \Aref{app:theory}). In \Cref{sec:experiments}, we empirically verify the scaling pattern of this bound. Two structural constants govern the bound: the within-domain operator norm $\beta_{w,\text{op}}^{(c),\text{cent}}$ on domain $c$'s pairwise-interaction block (centered by the within-domain mean $\bar\varphi_{c,c}$), and the cross-domain operator norm $\delta_{\text{op}}$ on $\Phi_{\text{res}}$. Global mixture identification (\Cref{ass:global-identification}) requires uniform predictor error below half the gap between the optimal mixture and all others. A finite-sample noise term $\mathcal{R}(m)$ scales linearly in $m$ and collects three sources of estimation error: the empirical residual matrix $\hat\Phi_{\text{res}}$, the predictor $\hat g$ fitted from probe runs, and its argmax (full breakdown in \Cref{def:R-of-m}). The M\"obius truncation tail $E(m)$ is the higher-order remainder from Eq.~(\ref{eq:decomposition}) (bounded in \Cref{thm:remainder}); it is a fixed structural quantity that does not vanish with more probe runs.

\begin{theorem}[Mixture-then-Select Spectral Guarantee]
\label{thm:spec-main}
Let $S^* = \arg\max_{|S|=m} u(S)$ and $S_{\text{MtS}}$ the output of Algorithm~\ref{alg:mg}. Under \Cref{ass:monotone,ass:within-op,ass:res-op,ass:decay} and global mixture identification (\Cref{ass:global-identification}), with probability $\geq 1 - \eta$ (over the probe-run randomness used to estimate $\hat\Phi$ and the predictor $\hat g$),
\begin{equation}
    u(S_{\text{MtS}}) \;\geq\; u(S^*) \;-\; \underbrace{\sum_{c=1}^T \beta_{w,\text{op}}^{(c),\text{cent}}\, n_c^*\, v_{\max}}_{\text{within-domain gap}} \;-\; \underbrace{\delta_{\text{op}}\, m\, v_{\max}}_{\text{cross-domain gap}} \;-\; \mathcal{R}(m) \;-\; 2\,E(m),
    \label{eq:spec-main}
\end{equation}
where $n_c^* = |S^* \cap C_c|$ is the optimum's per-domain count and $v_{\max} = \max_j v_j$ is the maximum singleton value.
\end{theorem}

Each term in Theorem~\ref{thm:spec-main} maps directly to a piece of the domain decomposition: the within-domain gap is the contribution of $\Phi_w$, the cross-domain gap $\delta_{\text{op}} m v_{\max}$ is the contribution of $\Phi_{\text{res}}$ (the only piece neither stage controls), $\mathcal{R}(m)$ collects finite-sample noise, and $E(m)$ is the higher-order M\"obius truncation tail. The within-domain gap charges Stage~2 for sub-optimal picks within each domain; each domain contributes only in proportion to its share $n_c^*$ of the optimum, so a heterogeneous domain spectrum tightens the bound below the uniform-domain worst case $\beta_{w,\text{op}}^{\text{cent}} m v_{\max}$. The cross-domain gap $\delta_{\text{op}} m v_{\max}$ charges Stage~1 for composition error against $\mathbf{n}^*$. Both gaps are linear in $m$ (rather than $m^2$) and do not grow with $T$.
The bound applies at sub-domain budgets ($m < |C_c|$); the constants are pessimistic. It describes scaling and estimation robustness rather than predicting the observed gap (numerical instantiation in \Aref{app:theory}).

The linear-in-$m$ scaling is unavoidable: the bound is tight up to a factor of $2$ on each axis, with matching lower-bound constructions in \Aref{app:theory} (\Cref{thm:tightness}).
Both \Cref{thm:spec-main} and \Cref{thm:tightness} assume access to the population pairwise interaction matrix $\Phi$. In practice $\Phi$ is estimated from a finite probe set; let the empirical estimate $\hat\Phi$ have operator-norm error $\|\hat\Phi - \Phi\|_{\text{op}} \leq \epsilon_{\text{op}}\, v_{\max}$. The bound picks up an additive penalty linear in $m$, matching the structural terms.

\begin{theorem}[Operator-Norm Robustness]
\label{thm:robust-op}
If $\|\hat\Phi - \Phi\|_{\text{op}} \leq \epsilon_{\text{op}} v_{\max}$, then Theorem~\ref{thm:spec-main} holds for $\hat\Phi$ with $\beta_{w,\text{op}}^{(c),\text{cent}}$ and $\delta_{\text{op}}$ replaced by $\beta_{w,\text{op}}^{(c),\text{cent}} + \epsilon_{\text{op}}$ and $\delta_{\text{op}} + \epsilon_{\text{op}}$, picking up an additive $2\,\epsilon_{\text{op}}\,m\,v_{\max}$ gap.
\end{theorem}

A symmetric robustness result for $v_j$-noise enters the bound additively (Corollary in \Aref{app:theory}): under uniform value-noise $\|\hat v - v\|_\infty \leq \epsilon_v\, v_{\max}$ ($\ell_\infty$ error), the bound picks up an additive $2\,\epsilon_v\, m\, v_{\max}$ gap, matching Theorem~\ref{thm:robust-op}'s structural form.
MtS extends to cost-weighted selection (data-marketplace deployment) via a one-line change to Stage~1's constraint, a price-weighted budget $\sum_c p_c\, n_c \le B$ replaces cardinality, with Theorem~\ref{thm:spec-main}'s bound carrying over to the realized cardinality (\Aref{app:theory}); sample-level baselines would require algorithmic redesign.
Three further results are deferred to \Aref{app:theory}: the adaptive IA-Greedy extension at $T{=}1$ with classical $(1{-}e^{-\gamma_{\text{op}}})$ guarantee, computational hardness (APX-hardness via MAX-CUT), and predictor sample complexity.

\section{Experiments and analysis}
\label{sec:experiments}

We compare performance across baselines and conditions (\Cref{sec:master_leaderboard}), account for computational efficiency (\Cref{sec:compute_equalized}), and analyze the mechanism behind the gains and where the method falls short (\Cref{sec:why_baselines_fail}).

\subsection{Setup}
\label{sec:setup}

\textbf{Data pool, models, metric.}
The data pool is a $50K$-sample instruction-tuning corpus from 13 public sources, partitioned into $T{=}8$ skill domains (math, code, QA, dialog, reasoning, NLP, complex-instruction, creative); we additionally evaluate on a $4{\times}$ larger $N{=}200K$ data pool. Qwen3-8B is the main base model; Llama-3.1-8B-Instruct (different architecture) and Llama-3.2-3B-Instruct (smaller scale) test cross-model transfer. All base models are LoRA fine-tuned at rank $r{=}16$. We measure per-domain held-out loss reduction $u_c(S)$ relative to the base model weighted on various tasks; held-out loss reduction is the standard target in recent instruction-tuning data-selection work~\citep{zhou2023lima,liu2024deita,xia2024less,li2024ifd} and tracks downstream benchmark gains in similar regimes. Selection sizes span $\sim$1\% to 20\% of the data pool. See detailed setups in \Aref{app:empirical}.

\textbf{Baselines.}
We compare MtS against eight baselines covering four method directions: \emph{per-sample quality} (LESS~\citep{xia2024less}, EL2N~\citep{paul2021deep}, Superfiltering~\citep{li2024superfiltering}, DEITA~\citep{liu2024deita}), \emph{diversity} ($k$-DPP~\citep{kulesza2011kdpp}, Facility Location~\citep{wei2015submodularity}), \emph{gradient matching} (CRAIG~\citep{mirzasoleiman2020coresets}), and \emph{partition-based} (Stratified, which allocates the budget uniformly across the $T$ skill domains and ranks within each by oracle; Random, which samples uniformly from the data pool with no per-domain structure), one representative per principle. Before selection, MtS runs a one-time per-base-model calibration that produces a per-sample value used for within-domain ranking; final LoRA fine-tuning is identical across methods. {We additionally compare with recent mixture-optimization methods DoReMi, RegMix and DIDS~\citep{xie2023doremi,liu2025regmix,shi2025dids}.}

\subsection{Performance comparisons}
\label{sec:master_leaderboard}
\label{sec:robustness_master}

We first verify that our data pool exhibits the positive sample interactions our framework requires, then compare MtS against baselines, and finally test robustness across model capacity, attribution method, and data pool size.

\textbf{Complementarity is real on the data pool.}
The structure our framework requires (positive sample interactions, $\beta>0$) is strong and consistent on this data pool, confirmed by two independent measurement protocols: a large random pair sweep, and a separate pilot study that averages over training seeds. Other empirical constants used by \Cref{thm:spec-main} are computed in \Aref{app:theory}.

\begin{table}[!t]
\vspace{-5pt}
\centering
\footnotesize
\setlength{\tabcolsep}{5pt}
\caption{Held-out utility comparison on Qwen3-8B and Llama-3.2-3B-Instruct, both fine-tuned with LoRA $r{=}16$. Higher utility is better. Per-cell seed counts and paired $t$-tests are in \Aref{app:empirical}.}
\label{tab:master_leaderboard}
\begin{tabular*}{\linewidth}{@{\extracolsep{\fill}} l l >{\centering\arraybackslash}p{1.7cm} >{\centering\arraybackslash}p{1.7cm} >{\centering\arraybackslash}p{1.7cm} @{}}
\toprule
\textbf{Selection} & \textbf{Attribution} & \textbf{Qwen3-8B} & \textbf{Qwen3-8B} & \textbf{Llama-3.2-3B} \\
                   &                      & $N{=}50K$         & $N{=}200K$        & $N{=}50K$ \\
                   &                      & $m{=}1K$          & $m{=}3.5K$        & $m{=}1K$ \\
\midrule
\multicolumn{5}{l}{\textit{\textbf{Panel~A.} Full leaderboard: MtS vs.\ 11 baselines, each with its native attribution.}} \\
\midrule
MtS & Oracle            & $\mathbf{0.694}$ & $\mathbf{0.718}$ & $\mathbf{0.281}$ \\
\addlinespace[2pt] \cmidrule{1-5} \addlinespace[2pt]
Stratified                     & Oracle            & $0.666$  & $0.694$ & $0.190$ \\
\addlinespace[2pt] \cmidrule{1-5} \addlinespace[2pt]
                                      & Oracle            & $0.603$  & $0.632$ & $-0.203$ \\
Greedy (global top-$m$)               & LESS                                & $0.548$  & $0.580$           & $-0.238$ \\
                                      & EL2N                                & $0.567$  & $0.598$ & $-0.244$ \\
                                      & Superfiltering                      & $0.534$  & $0.640$           & $-0.179$ \\
                                      & DEITA                               & $0.531$  & $0.579$           & $-0.109$ \\
\addlinespace[2pt] \cmidrule{1-5} \addlinespace[2pt]
$k$-DPP                               & DPP kernel                          & $0.581$            & $0.612$           & $-0.101$ \\
Facility Location                     & Submodular surrogate                & $0.554$            & $0.586$           & $-0.016$ \\
Random                                & Uniform                             & $0.549$  & $0.567$           & $0.076$ \\
CRAIG                                 & Gradient matching                   & $0.452$  & $0.579$           & $-0.204$ \\
\addlinespace[2pt] \cmidrule{1-5} \addlinespace[2pt]
Full data (no selection)              & full-retraining                     & $0.548$  & $0.581$           & $0.030$ \\
\specialrule{0.5pt}{4pt}{1pt}
\specialrule{0.5pt}{1pt}{4pt}
\multicolumn{5}{l}{\textit{\textbf{Panel~B.} Same $3$ selection rules under cheap attributions (LESS, EL2N); Oracle in Panel~A.}} \\
\midrule
MtS                                   & LESS                    & $0.591$      & $0.625$           & $\mathbf{-0.120}$ \\
                                      & EL2N                    & $\mathbf{0.597}$      & $\mathbf{0.677}$ & $-0.122$ \\
\addlinespace[2pt] \cmidrule{1-5} \addlinespace[2pt]
Stratified                            & LESS                    & $0.576$      & $0.615$           & $-0.132$ \\
                                      & EL2N                    & $0.567$      & $0.639$ & $-0.137$ \\
\addlinespace[2pt] \cmidrule{1-5} \addlinespace[2pt]
Greedy (global top-$m$)               & LESS                    & $0.548$                & $0.580$           & $-0.238$ \\
                                      & EL2N                    & $0.567$                & $0.598$ & $-0.244$ \\
\bottomrule
\end{tabular*}
\vspace{-8pt}
\end{table}

\textbf{MtS outperforms every baseline and the lead is robust.}
{In the Qwen3-8B setting with a 50K pool and 1K selection budget, MtS improves utility over Stratified and oracle-scored global top-$m$ selection by $4.2\%$ and $15.1\%$, respectively (Table~\ref{tab:master_leaderboard}).} The eight sample-level methods bunch together near random, indifferent to which selection principle they implement. MtS's $1{,}000$-sample subset also beats training on the full $50K$ data pool by $\Delta{=}{+}0.146$, a small, well-chosen subset wins over training on everything. The rightmost column flags a striking failure mode: on Llama-3.2-3B, every sample-level method goes negative if Qwen's selections are reused without re-running calibration on Llama itself. We return to the diagnosis in \Cref{sec:why_baselines_fail}.

{\textbf{Comparison with recent mixture methods.} MtS reaches 0.694 held-out utility, compared with 0.6566 for RegMix and 0.6272 for DoReMi, relative gains of 5.7\% and 10.7\%. RegMix and DoReMi use a 1K subset with the same within-domain selection rule. The paired comparison with RegMix gives $p=0.0117$. DIDS reaches 0.6343 by sampling online from the full pool with the same total training exposures.\par}

\begin{table}[t]
\vspace{-5pt}

\centering
\footnotesize
\setlength{\tabcolsep}{1.5pt}
\captionsetup{font={color=black},labelfont={color=black}}
\caption{Downstream benchmarks on Qwen3-8B with the 50K pool. Metrics are instruction-level strict accuracy for IFEval, accuracy for MMLU and GSM8K, normalized accuracy for ARC-C, and pass@1 for HumanEval.}
\label{tab:rebuttal_benchmarks_main}
\begin{tabular}{@{}lrrrrr@{}}
\toprule
Method & IFEval & MMLU & GSM8K & ARC-C & \shortstack{Human\\Eval} \\
\midrule
MtS & {\textbf{0.501}} & {\textbf{0.722}} & {\textbf{0.803}} & {\textbf{0.575}} & {\textbf{0.795}} \\
RegMix & 0.482 & 0.714 & 0.796 & 0.556 & 0.770 \\
DoReMi & 0.480 & 0.715 & 0.802 & 0.532 & 0.762 \\
{DSIR} & {0.494} & {0.714} & {0.785} & {0.526} & {0.747} \\
LESS & 0.448 & 0.721 & 0.793 & 0.549 & 0.751 \\
\bottomrule
\end{tabular}
\vspace{-8pt}
\end{table}

{\noindent\textbf{Downstream benchmark evaluation.} Table~\ref{tab:rebuttal_benchmarks_main} reports five benchmarks. MtS has the highest reported score on all five tasks in this panel, with the clearest gains on IFEval and ARC-Challenge. These results assess downstream performance alongside held-out utility, whose ranking need not match accuracy on every task.}\par\color{black}

\textbf{Gains grow with model capacity.}
MtS ties Stratified at the lowest tested LoRA rank, but its advantage grows monotonically with rank, reaching $\Delta{=}{+}0.046$ at $r{=}32$ on Qwen3-8B (Figure~\ref{fig:main_diagnostic}B). This pattern suggests that higher capacity helps exploit cross-domain interactions; LESS remains near random across ranks (\Aref{app:empirical}). \textbf{Mixture matters more than attribution.} The ranking Greedy~$<$~Stratified~$<$~MtS holds in every attribution column we tested, on both Qwen3-8B and Llama-3.2-3B. Replacing the expensive oracle attribution with cheap EL2N recovers $97\%$ of MtS's oracle utility, at roughly four orders of magnitude less compute. The choice of attribution mainly trades cost for accuracy; the load-bearing decision is the domain mixture.

\textbf{Gains transfer to larger data pools and other architectures.}
The advantage does not erode at $4\times$ the data pool size, and transfers to a different model family (Llama-3.1-8B-Instruct), with statistically significant gains over Stratified in both settings; on the smaller Llama-3.2-3B, the gap further widens as the budget grows (full numbers in Table~\ref{tab:robustness_profile} and \Aref{app:empirical}). The lead also persists when the eight skill domains are coarsened into four macro-domains ($\Delta{=}{+}0.022$ at $T{=}4$, same direction as $T{=}8$; \Aref{app:domain_count_sensitivity}). {Evaluation on a 72B model further supports MtS's applicability across model scales (\Aref{app:transfer_72b}, Table~\ref{tab:rebuttal_72b_main}).}

\WFclear
\subsection{Computational efficiency}
\label{sec:compute_equalized}

\begin{wraptable}{R}{0.46\textwidth}
\vspace{-5pt}
\color{black}
\centering
\footnotesize
\caption{Compute and utility on Qwen3-8B with $N{=}50K$ and LoRA $r{=}16$. Total GPU-hours include selection and training over five seeds; the breakdown is in \Cref{tab:compute_equalized}.}
\label{tab:compute_equalized_main}
\begin{tabular*}{\linewidth}{@{} l @{\extracolsep{\fill}} c c c @{}}
\toprule
\textbf{Method} & $\boldsymbol{m}$ & $\boldsymbol{u}$ & \textbf{GPU-hr} \\
\midrule
Stratified              & 1k  & $0.666$ & $0.58$  \\
Stratified              & 5k  & $0.678$ & $2.92$  \\
EL2N                           & 1k  & $0.567$ & $1.43$  \\
\textbf{MtS}          & \textbf{1k}  & $\mathbf{0.694}$ & $\mathbf{16.58}$  \\
MtS                   & 5k  & $0.691$ & $18.92$ \\
\textbf{Full data}             & \textbf{50k} & $\mathbf{0.548}$ & $\mathbf{29.17}$ \\
\textbf{LESS}                  & \textbf{1k}  & $\mathbf{0.548}$ & $\mathbf{70.58}$ \\
\bottomrule
\end{tabular*}
\vspace{-8pt}
\end{wraptable}

\textbf{MtS costs less than full-data training and LESS, including calibration.} In \Cref{tab:compute_equalized_main}, it uses 43\% less total compute than full-data training and 77\% less than LESS, with approximately 27\% higher utility. No listed baseline uses less compute while achieving higher utility than MtS at $m{=}1000$. Stratified remains cheaper for a one-off selection, but at $m{=}5000$ reaches 0.678 utility versus MtS's 0.694. Increasing MtS's own budget to 5K does not improve utility. The same calibration can serve multiple selection settings on a fixed pool and base model. At fixed target accuracy, the required number of probes grows linearly with domain count and predictor rank, independently of pool size (\Cref{cor:margin-empirical}). Cost breakdowns and reuse across settings and models are in \Aref{app:calibration_reuse}.\par\color{black}

\WFclear
\subsection{Analysis of mechanism}
\label{sec:why_baselines_fail}
\label{sec:boundary_regimes}
\label{sec:experiments_summary}

We turn to two questions: what drives MtS's outperformance, and how does its edge vary across regimes? Practitioner takeaways close the section.

\textbf{Sample-level scores miss interaction structure.}
Ranking samples globally can concentrate the selection in high-scoring domains while missing useful combinations across domains. Figure~\ref{fig:main_diagnostic}A shows this contrast: MtS concentrates on dialog/creative/QA, while Greedy-by-value concentrates on math/NLP, nearly the opposite mixture. Several baselines, including LESS and CRAIG, remain near random utility despite their different selection criteria (Table~\ref{tab:master_leaderboard}). The controlled comparisons using the same attribution scores also favor MtS, showing that the selection rule matters beyond the scores themselves. When samples interact, how the chosen samples combine matters more than how good each one looks alone.

\begin{figure}[t]
\centering
\includegraphics[width=\linewidth]{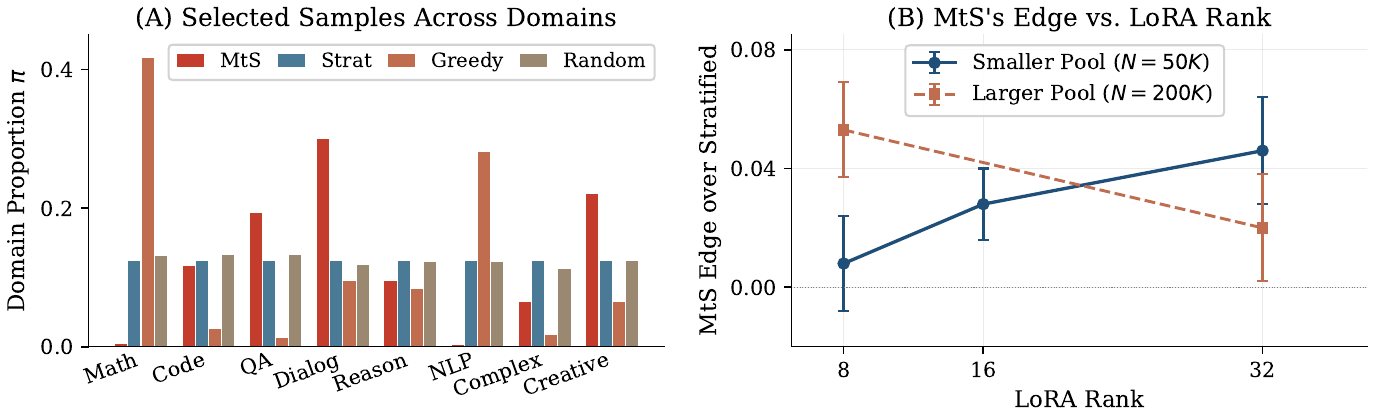}
\caption{Domain allocation and capacity on Qwen3-8B. \textbf{(A)} Domain proportions selected by four selection rules at $N{=}50K$, $m{=}1K$. \textbf{(B)} MtS's utility difference from Stratified across LoRA ranks, with $m{=}1K$ at $N{=}50K$ and $m{=}3.5K$ at $N{=}200K$.}
\label{fig:main_diagnostic}
\vspace{-12pt}
\end{figure}

{\textbf{Empirical support for pairwise modeling.}
\label{sec:interaction_validation}
Singleton and pairwise terms explain 90\% of utility variation across approximately 100 measured triplets. Adding pairwise terms reduces mean absolute prediction error from 0.014 to 0.004, approximately a 71\% reduction. Additional measurements and cross-task consistency checks are reported in \Aref{app:pairwise_validation}. \textbf{Mixture and attribution information both contribute.} The set-level variant reaches 0.6927 utility when its predictor uses both domain counts and sample-value statistics. Keeping only domain counts lowers utility by 0.088; keeping only value statistics lowers it by 0.063 (\Aref{app:ablations}). These are separate predictor-input ablations, not removals of selection stages or additive contributions. Together with the Greedy-Stratified-MtS ordering across attribution signals, these results show that both types of information matter. We also analyze settings where MtS does not improve over Stratified (\Aref{app:robustness_profile}).\par}


\textbf{Practitioner takeaways.} EL2N offers a cheaper alternative to oracle attribution. When transferring a learned mixture to another model, recompute within-domain rankings for that model. A possible screening rule is to fit a predictor on a small set of probes and use MtS when its predicted gain over uniform allocation exceeds the prediction uncertainty; otherwise use Stratified. The budget and capacity ablations motivate this rule, but we have not tested it separately.

{\textbf{Robustness to domain definitions.} MtS retains gains when domains are merged or samples are reassigned. Merging eight domains into four and replacing source labels with utility-derived clusters retain utility advantages of 3.5\% and 3.2\% over Stratified, respectively, compared with 4.3\% under the original partition. Corrupting 20\% of labels reduces the four-task downstream composite by only 0.7\% across three seeds. When 30\% of samples have two possible domain labels, reassigning them using both embeddings and utility vectors increases Jaccard overlap with the clean selection from 0.403 to 0.806. In a separate end-to-end test at 30\% ambiguity with embedding-based reassignment, the lambda-0.50 variant improves utility by 3.09\% and the four-benchmark mean by 2.42\% over Stratified; both gains occur in all three seeds. Settings, selection overlap and downstream results appear in \Aref{app:domain_robustness}.\par}

\par\medskip
\noindent\begin{minipage}{\linewidth}
\begin{minipage}[t]{0.50\linewidth}
\vspace{0pt}
\noindent\textbf{Data pricing.}\label{sec:pricing_pilot} MtS also selects training data for purchase under a fixed budget. We test a vendor-bundle market~\citep{zhang2025fairshare} with $15$ indivisible $50$-sample bundles across math, code and general QA, averaging over $10$ price/utility shuffles. Budgets are $15\%$, $20\%$ or $25\%$ of total cost. MtS uses the same two-stage selection under a price budget. Additive LP instead allocates the budget using each cluster's utility alone, ignoring interactions. Additive LP loses utility as the budget grows, while MtS stays close to the oracle (\Cref{tab:pricing_pilot}): it closes at least $89\%$ of the gap between Additive LP and the oracle at every budget and matches the oracle on $58\%$ of instances. Per-instance results and ablations are in \Aref{app:pricing_pilot}.\par\color{black}
\end{minipage}\hfill
\begin{minipage}[t]{0.46\linewidth}
\vspace{0pt}
\color{black}
\centering
\footnotesize
\setlength{\tabcolsep}{1.9pt}
\captionof{table}{Data purchasing under a budget on Llama-3.2-3B-Instruct, averaged over 10 pricing instances. $B$ is the fraction of total cost available for purchase. Recovery measures the fraction of the Additive-LP-to-oracle utility gap closed by MtS. Setup: \Aref{app:pricing_pilot}.}
\label{tab:pricing_pilot}
\begin{tabular}{@{}lccc@{}}
\toprule
\textbf{Method} & $B{=}15\%$ & $B{=}20\%$ & $B{=}25\%$ \\
\midrule
Additive LP                  & $0.197$ & $0.146$ & $0.137$ \\
\textbf{MtS}                 & $\mathbf{0.229}$ & $\mathbf{0.242}$ & $\mathbf{0.249}$ \\
Oracle                       & $0.236$ & $0.246$ & $0.250$ \\
\midrule
$\Delta$ MtS$-$Additive LP   & $0.033$ & $0.097$ & $0.111$ \\
Optimality recovery          & $89.5\%$ & $95.6\%$ & $98.3\%$ \\
\bottomrule
\end{tabular}
\end{minipage}
\end{minipage}
\par\medskip
\WFclear

\section{Conclusion}
\label{sec:conclusion}

Combinatorial Data Selection separates scoring individual samples from choosing a subset whose samples work well together. Mixture-then-Select first chooses domain counts, then ranks samples within each domain. Experiments show gains across models and pool sizes, including a setting where training data must be purchased under a budget. Future work includes aligning utility with specific downstream tasks and reducing calibration cost.

\subsection*{Reproducibility statement}

{The appendix provides the theoretical assumptions and proofs,
experimental settings, and detailed results needed to reproduce the study.\par}

\bibliography{references}
\bibliographystyle{iclr2027_conference}

\newpage
\appendix
\renewcommand{\thesection}{\Alph{section}}
\renewcommand{\thesubsection}{\thesection.\arabic{subsection}}

\section{Theory and proofs}
\label{app:theory}

\subsection{Notation table}
\label{app:notation}

\begin{table}[h]
\vspace{-5pt}
\centering
\small
\begin{tabular}{ll}
\toprule
\textbf{Symbol} & \textbf{Description} \\
\midrule
$\mathcal{D} = \{1, \ldots, N\}$ & Data catalog \\
$p_j$ & Cost of sample $j$ \\
$B$ & Selection budget \\
$u(S)$ & True utility of subset $S$ \\
$v_j = u(\{j\})$ & Individual value of sample $j$ \\
$\varphi_{ij}$ & Pairwise interaction ($<0$: substitutes, $>0$: complements) \\
{$\tilde u(S)$} & Pairwise approximation: $\sum v_j + \sum \varphi_{ij}$ \\
$\varepsilon(S)$ & Higher-order remainder: $u(S) - {\tilde u(S)}$ \\
{$\hat g(\mathbf n)$} & {Learned allocation-utility predictor in Algorithm~\ref{alg:mg}} \\
{$\hat F(S)$} & {Learned set-level regressor for D1 and empirical allocation} \\
{$\hat\varphi_{ij}$} & {Predicted pair interaction (\Aref{app:proxy_quality})} \\
$\Delta_f(s \mid L)$ & Marginal gain of $s$ given $L$ \\
$\gamma$ & Submodularity ratio \\
$\beta$ & Complementarity bound in $\mathrm{BC}(\beta)$ \\
$C, \rho_d$ & Decay constants for higher-order interactions \\
$v_{\max}$ & Maximum individual value: $\max_j v_j$ \\
$n_{\max}$ & Maximum selection size \\
$\rho_s$ & Spearman rank correlation \\
\midrule
\multicolumn{2}{l}{\emph{Spectral form (\Cref{sec:formulation,sec:theory}, \Cref{thm:spec-main})}} \\
$\Phi_w, \Phi_{\text{block}}, \Phi_{\text{res}}$ & Within-domain, block-constant cross, cross-residual parts of $\Phi$ \\
$\beta_{w,\text{op}}^{(c),\text{cent}}$ & Per-domain centered within-block operator-norm coefficient \\
$\delta_{\text{op}}$ & Cross-residual operator-norm coefficient \\
$\nu_{\text{op}}^\perp$ & Empirical-to-population sub-matrix concentration constant on $\hat\Phi_{\text{res}}$ \\
$\zeta_{\text{pred}}, \zeta_\Delta^{\text{loc}}$ & Predictor sample-complexity gap; predictor-argmax displacement on the local-margin ball \\
$\rho_0, \mu_{\text{loc}}$ & Local-margin radius and rate (Corollary~\ref{cor:margin-empirical}) \\
$L, r, T$ & Number of labeled subsets; predictor rank; domain count \\
$S_{\text{MtS}}, S_{\text{IA}}, S_{\text{HMtS}}$ & Mixture-then-Select, IA-Greedy, hybrid output ($\arg\max\{u(S_{\text{MtS}}), u(S_{\text{IA}})\}$) \\
\bottomrule
\end{tabular}
\vspace{-8pt}
\end{table}

\subsection{Formal assumptions and operator-norm refinement}
\label{app:assumptions}

This appendix collects the formal assumptions on the utility $u$ and the pairwise interaction matrix $\Phi$ that drive the bounds in \Cref{sec:spec_main}. \Aref{app:assumptions} also states Proposition~\ref{prop:complementarity}, which decomposes complementarity into a model-induced and an aggregation-induced source, and discusses why the operator-norm refinement (\Cref{ass:within-op,ass:res-op}) sharpens the entrywise BC($\beta$) form by a factor of $m$.

\subsection{Core assumptions}

Three assumptions cover the basic shape of the utility: monotonicity of each task's improvement, bounded complementarity of pairwise interactions, and decay of higher-order interactions.

\begin{assumption}[Monotonicity]
\label{ass:monotone}
Each $u_t$ is monotone non-decreasing: $u_t(S \cup \{j\}) \geq u_t(S)$ for all $S, j, t$.
\end{assumption}

\begin{assumption}[Bounded Complementarity $\mathrm{BC}(\beta)$]
\label{ass:bc}
For each $j \in \mathcal{D}$ and any $R \subseteq \mathcal{D} \setminus \{j\}$,
$\sum_{i \in R} \varphi_{ij}^{+} \leq \beta \cdot v_j,$
where $\varphi_{ij}^{+} = \max(\varphi_{ij}, 0)$ denotes the complementary (positive) part of the aggregate interaction.
\end{assumption}

\begin{assumption}[Exponential Interaction Decay]
\label{ass:decay}
Higher-order M\"obius coefficients satisfy $|a_R| \leq C \cdot \rho_d^{|R|} \cdot v_{\max}$ for $|R| \geq 3$, where $C, \rho_d > 0$ are constants and $v_{\max} = \max_j v_j$.
\end{assumption}

Monotonicity (Assumption~\ref{ass:monotone}) is the standard hypothesis in submodular data selection~\citep{wei2015submodularity,mirzasoleiman2020coresets,nemhauser1978analysis}: adding training data never hurts. Since $f$ is non-decreasing in each argument, $u$ inherits monotonicity from the $u_t$'s for every aggregator considered in this paper.

BC($\beta$) (Assumption~\ref{ass:bc}) controls positive pairwise interactions relative to singleton values. For the pairwise approximation, $\beta=0$ gives the substitution-only case; a positive submodularity-ratio guarantee additionally requires the conditional-marginal hypothesis in Lemma~\ref{thm:gamma}.

Assumption~\ref{ass:decay} controls the higher-order M\"obius tail. M\"obius inversion~\citep{grabisch1999axiomatic} writes any set function as singleton, pairwise, triplet, $\ldots$ contributions, and truncating after pairs is justified only when the dropped terms are small; the geometric decay condition is the quantitative form of that requirement. On our data pool the third-order remainder is one to two orders of magnitude below the pairwise term (\Cref{sec:experiments}).

\subsection{Two sources of complementarity}

Whether $\beta > 0$ comes from the per-task utilities being non-submodular or from the aggregator $f$ matters for interpreting empirical $\beta$ calibrations. The next proposition separates the two sources.

\begin{proposition}[Two Sources of Complementarity]
\label{prop:complementarity}
Let each $u_t$ be submodular.
\begin{enumerate}
    \item[(a)] \textbf{Linear aggregation preserves submodularity.} If $f = \sum_t w_t \cdot u_t$ (weighted average) and each task-specific utility $u_t$ is submodular, then $u$ is submodular and $\beta = 0$.
    \item[(b)] \textbf{Concave aggregation can induce complementarity.} If $f = \min_t u_t$ or $f$ is strictly concave, then $\beta > 0$ whenever samples from different domains cover different tasks, even when all per-task interactions are substitutive ($\varphi_{ij}^{(t)} \leq 0$).
\end{enumerate}
\end{proposition}

Proposition~\ref{prop:complementarity} separates two sources of complementarity in multi-task data selection:
\begin{itemize}
    \item \textbf{Source 1: Model-induced.} The per-task utilities $u_t$ are themselves non-submodular, training on $i$ and $j$ together produces more improvement \emph{on task $t$ alone} than their individual contributions sum to. This arises when the model develops cross-task representations (e.g., a code sample improves math reasoning through shared compositional skills). Under linear $f$, the contrapositive of part (a) makes this the only source of $\beta > 0$.
    \item \textbf{Source 2: Aggregation-induced.} Even when every $u_t$ is submodular ($\varphi_{ij}^{(t)} \leq 0$ for all $i,j,t$), a concave $f$ creates aggregate complementarity by rewarding balance across tasks (part (b)).
\end{itemize}
Our experiments use linear aggregation $f = \frac{1}{T}\sum_t u_t$, which shuts off Source~2; any observed $\beta > 0$ is therefore evidence that the per-task utilities themselves are non-submodular, a property of the model's learning dynamics that any downstream objective inherits. The framework supports concave aggregators (Appendix~\ref{app:aggregators}); the linear choice here isolates Source~1.

\subsection{Operator-norm refinement}

BC($\beta$) bounds the interactions one entry at a time, which forces an $O(m^2)$ gap and is loose at production budgets. MtS actually controls interactions collectively: matching the domain composition cancels the cross-domain block mean exactly, so what remains is the residual matrix $\Phi_{\text{res}}$ acting on the indicator $\mathbf{1}_S$. A single Cauchy-Schwarz step bounds that action by the operator norm of $\Phi_{\text{res}}$ times $\sqrt{|S|}$, which delivers an $O(m)$ gap on any matrix whose mass is not concentrated on a single direction aligned with $\mathbf{1}_S$. The next two assumptions name the operator norms that quantify this control.

\begin{assumption}[Bounded Within-Domain Operator Norm; B1''-cent]
\label{ass:within-op}
For each domain $c$, the centered within-domain block $\Phi_{w,c}^{\text{cent}} := \Phi_{w,c} - \bar\varphi_{c,c}(\mathbf{1}\mathbf{1}^\top - I)$ satisfies $\|\Phi_{w,c}^{\text{cent}}\|_{\text{op}} \leq \beta_{w,\text{op}}^{(c),\text{cent}}\, v_{\max}$ at the population level. When the domain index is suppressed, $\beta_{w,\text{op}}^{\text{cent}} := \max_c \beta_{w,\text{op}}^{(c),\text{cent}}$; at $T{=}1$ the two coincide.
\end{assumption}

\begin{assumption}[Bounded Cross-Residual Operator Norm; B2'']
\label{ass:res-op}
The population cross-residual satisfies $\|\Phi_{\text{res}}\|_{\text{op}} \leq \delta_{\text{op}}\, v_{\max}$. The empirical-to-population gap on the realized subset is controlled by a sub-matrix concentration bound (Appendix~\ref{app:robust-op}, Lemma M').
\end{assumption}

The operator norm bounds a matrix's largest action on any vector relative to that vector's $\ell_2$-norm; applied to the indicator $\mathbf{1}_S$ (which has norm $\sqrt{|S|}$) this gives $|\mathbf{1}_S^\top \Phi\, \mathbf{1}_S| \leq \|\Phi\|_{\text{op}}\,|S|$, exactly the worst-case spectral loss the algorithm can incur on a subset of size $|S|$. The constants $\beta_{w,\text{op}}^{(c),\text{cent}}$ and $\delta_{\text{op}}$ are this worst-case action, normalized by $v_{\max}$, on $\Phi_{w,c}^{\text{cent}}$ and $\Phi_{\text{res}}$ respectively. Both are intrinsic to the data pool and partition, are estimated from $\hat\Phi$, and the empirical-to-population gap is charged in Theorem~\ref{thm:robust-op}. The within-domain norm is taken on the centered block $\Phi_{w,c}^{\text{cent}}$ because Stage~1 of MtS matches $\bar\varphi_{c,c}$ exactly through the composition; the leftover deviation is what the bound has to pay for.

\begin{remark}
The adaptive $T=1$ extension uses conditional-marginal control (Lemma~\ref{thm:gamma}); \Cref{ass:within-op,ass:res-op,ass:global-identification} drive the fixed-ranking multi-domain bound (Theorem~\ref{thm:spec-main}).
\end{remark}

\subsection{Local-margin condition and finite-sample noise}
\label{app:local-margin}

\begin{definition}[Local-margin condition]
\label{def:local-margin}
There exist a \emph{local radius} $\rho_0 > 0$ and a \emph{local rate} $\mu_{\text{loc}} > 0$ such that for every mixture $\mathbf{n}$ within $\ell_1$-radius $\rho_0$ of the optimum mixture $\mathbf{n}^*$,
\[
   u(\mathbf{n}^*) - u(\mathbf{n}) \;\geq\; \mu_{\text{loc}}\, v_{\max}\, \|\mathbf{n} - \mathbf{n}^*\|_1 \,/\, m.
\]
\end{definition}

Here $u(\mathbf n)$ denotes the constrained optimum $U(\mathbf n)$ defined below. A local margin controls displacement within the specified ball; the main guarantee uses the following global condition to identify the maximizing mixture.

\begin{assumption}[Global mixture identification]
\label{ass:global-identification}
Let $\mathcal N_m=\{\mathbf n\in\mathbb Z_{\geq0}^T:\sum_c n_c=m,\ n_c\leq |C_c|\}$ and define
\[
U(\mathbf n)=\max_{S:\,\mathbf n(S)=\mathbf n}u(S).
\]
The optimum mixture $\mathbf n^*$ is unique, with gap
\[
\Delta_{\mathrm{mix}}
=U(\mathbf n^*)-\max_{\mathbf n\in\mathcal N_m\setminus\{\mathbf n^*\}}U(\mathbf n)>0.
\]
With probability at least $1-\eta$, the Stage~1 predictor satisfies
\[
\sup_{\mathbf n\in\mathcal N_m}|\hat g(\mathbf n)-U(\mathbf n)|
\leq \epsilon_{\mathrm{mix}}<\Delta_{\mathrm{mix}}/2.
\]
For a singleton feasible mixture set, take $\Delta_{\mathrm{mix}}=\infty$.
\end{assumption}

This is a global accuracy and separation assumption for the constrained optimum $U$, not a consequence of fitting the observed probes. It ensures exact mixture identification without changing Stage~1's search. A local margin alone controls displacement only after localization has been established.

\begin{definition}[Finite-sample noise term $\mathcal{R}(m)$]
\label{def:R-of-m}
The nonnegative estimation allowance in Theorem~\ref{thm:spec-main} is written as
\[
    \mathcal{R}(k) \;=\; \nu_{\text{op}}^\perp\, k\, v_{\max} \;+\; \zeta_\Delta^{\text{loc}}(\rho_0, \mu_{\text{loc}}) \;+\; 2\,\zeta_{\text{pred}}\, v_{\max},
\]
where $\nu_{\mathrm{op}}^\perp$ bounds residual estimation error, $\zeta_\Delta^{\mathrm{loc}}$ denotes an allowance for mixture displacement, and $\zeta_{\mathrm{pred}}$ bounds normalized predictor error. Under \Cref{ass:global-identification}, the mixture is recovered exactly and $\zeta_\Delta^{\mathrm{loc}}=0$. The structural constants $\beta_{w,\mathrm{op}}^{(c),\mathrm{cent}}$ and $\delta_{\mathrm{op}}$ concern population interactions; $\epsilon_{\mathrm{op}}=\|\hat\Phi-\Phi\|_{\mathrm{op}}/v_{\max}$ is a separate estimation-error budget used in \Cref{thm:robust-op}.
\end{definition}

\subsection{NP-hardness}
\label{app:proofs}

\begin{theorem}[NP-Hardness of CDS]
\label{thm:hardness}
(a) CDS is NP-hard even when $\varphi_{ij} = 0$ for all $i, j$.
(b) Unconstrained CDS with $\varphi_{ij} \leq 0$ for all $i, j$ is NP-hard.
\end{theorem}

\textbf{Part (a).} When $\varphi_{ij} = 0$ for all $i, j$, CDS reduces to: maximize $\sum_{j \in S} v_j$ subject to $\sum_{j \in S} p_j \leq B$, which is the 0-1 Knapsack problem~\citep{karp1972reducibility}. Since Knapsack is NP-hard, so is CDS. \qed

\textbf{Part (b).} Given a weighted graph $G = (V, E, w)$ with MAX-CUT instance, define: $v_j = \deg_w(j) = \sum_i w_{ij}$ for each vertex $j$, and $\varphi_{ij} = -2w_{ij}$ for each edge. Note $\varphi_{ij} \leq 0$ (substitution only). Then:
\begin{align*}
{\tilde u(S)} &= \sum_{j \in S} \deg_w(j) - 2\sum_{\{i,j\} \subseteq S} w_{ij} \\
&= \sum_{j \in S} \sum_{i \in V} w_{ij} - 2\sum_{\{i,j\} \subseteq S} w_{ij} \\
&= \sum_{j \in S} \Big(\sum_{i \in V \setminus S} w_{ij} + \sum_{i \in S} w_{ij}\Big) - 2\sum_{\{i,j\} \subseteq S} w_{ij} \\
&= \sum_{j \in S} \sum_{i \in V \setminus S} w_{ij} = \mathrm{cut}(S, V \setminus S).
\end{align*}
Since MAX-CUT is NP-hard~{\citep{karp1972reducibility}}, unconstrained CDS with substitution-only interactions is NP-hard. \qed

\subsection{Submodularity under substitution}
\label{app:submodular}

\begin{lemma}
\label{lem:submodular}
If $\varphi_{ij} \leq 0$ for all $i, j$ (substitution only), then ${\tilde u}$ is submodular.
\end{lemma}

When $\varphi_{ij} \leq 0$ for all $i, j$, we show ${\tilde u}$ has non-increasing marginal gains. For any $A \subseteq B$ and $s \notin B$:
\begin{align*}
\Delta_{{\tilde u}}(s \mid A) &= v_s + \sum_{i \in A} \varphi_{is}, \\
\Delta_{{\tilde u}}(s \mid B) &= v_s + \sum_{i \in B} \varphi_{is} = v_s + \sum_{i \in A} \varphi_{is} + \sum_{i \in B \setminus A} \varphi_{is}.
\end{align*}
Since $\varphi_{is} \leq 0$ for all $i$, the additional terms $\sum_{i \in B \setminus A} \varphi_{is} \leq 0$. Thus $\Delta_{{\tilde u}}(s \mid B) \leq \Delta_{{\tilde u}}(s \mid A)$, confirming submodularity. \qed

\subsection{Higher-order remainder bound}
\label{app:remainder}

\begin{theorem}[Higher-Order Remainder Bound]
\label{thm:remainder}
Under Assumption~\ref{ass:decay}, for any $S$ with $|S| = n$ and $n\rho_d < 1$:
$|\varepsilon(S)| \leq C \cdot v_{\max} \cdot (n\rho_d)^3 / [6(1 - n\rho_d)]$.
\end{theorem}

Under Assumption~\ref{ass:decay}, for $|T| = m \geq 3$: $|a_T| \leq C \rho_d^m v_{\max}$. For $|S| = n$:
\begin{align*}
|\varepsilon(S)| &= \Big|\sum_{\substack{T \subseteq S \\ |T| \geq 3}} a_T\Big| \leq \sum_{k=3}^{n} \binom{n}{k} C \rho_d^k v_{\max} \\
&= C v_{\max} \sum_{k=3}^{n} \binom{n}{k} \rho_d^k \leq C v_{\max} \sum_{k=3}^{\infty} \frac{(n\rho_d)^k}{k!} \\
&\leq C v_{\max} \cdot \frac{(n\rho_d)^3}{6} \sum_{k=0}^{\infty} (n\rho_d)^k = \frac{C v_{\max} (n\rho_d)^3}{6(1-n\rho_d)},
\end{align*}
where the last step uses the geometric series bound (valid for $n\rho_d < 1$) and $\binom{n}{m} \leq n^m / m!$. \qed

\subsection{Submodularity ratio under conditional bounded complementarity}
\label{app:gamma}

The adaptive guarantee requires control of conditional gains, not only the singleton-based condition in \Cref{ass:bc}. Write $\Delta_{\tilde u}(s\mid L)=\tilde u(L\cup\{s\})-\tilde u(L)$.

\begin{lemma}[Submodularity ratio under conditional bounded complementarity]
\label{thm:gamma}
Suppose the pairwise utility $\tilde u$ is monotone on sets of size at most $2m$. For every disjoint $L,S$ with $|L|<m$, $|S|\leq m$, and each $s\in S$, assume
\[
\sum_{s'\in S\setminus\{s\}}\varphi_{ss'}^+
\leq \beta\,\Delta_{\tilde u}(s\mid L).
\]
Then the submodularity ratio over these pairs satisfies $\gamma\geq 2/(2+\beta)$.
\end{lemma}

\begin{proof}
Set $\mathrm{NUM}=\sum_{s\in S}\Delta_{\tilde u}(s\mid L)$ and $\Phi_S=\sum_{\{s,s'\}\subseteq S}\varphi_{ss'}$. The pairwise expansion gives
\[
\mathrm{DEN}=\tilde u(L\cup S)-\tilde u(L)=\mathrm{NUM}+\Phi_S.
\]
Summing the assumed conditional bound over $s$ counts each positive pair twice, so $\Phi_S\leq(\beta/2)\mathrm{NUM}$. Hence $\mathrm{DEN}\leq(1+\beta/2)\mathrm{NUM}$. For positive joint gain this gives $\mathrm{NUM}/\mathrm{DEN}\geq2/(2+\beta)$; zero joint gain imposes no restriction.
\end{proof}

The additional condition reduces to a singleton-based bound when $L=\emptyset$, but must hold after conditioning as well. It is used only for the adaptive extension, not for the fixed-ranking guarantee in \Cref{thm:spec-main}.

\subsection{Proof of Theorem~\ref{thm:spec-main} (Mixture-then-Select spectral guarantee)}
\label{app:spec-main}

We prove the bound for Algorithm~\ref{alg:mg}'s fixed within-domain ranking. Condition on the uniform-accuracy event in \Cref{ass:global-identification}, which has probability at least $1-\eta$.

\textbf{Mixture identification.}
For every feasible $\mathbf n\neq\mathbf n^*$,
\[
\hat g(\mathbf n^*)-\hat g(\mathbf n)
\geq U(\mathbf n^*)-U(\mathbf n)-2\epsilon_{\mathrm{mix}}
\geq\Delta_{\mathrm{mix}}-2\epsilon_{\mathrm{mix}}>0.
\]
Thus Stage~1 selects $\hat{\mathbf n}=\mathbf n^*$. In particular, $S_{\mathrm{MtS}}$ and $S^*$ have the same counts in every domain, and the composition-displacement term is zero on this event.

\textbf{Fixed-ranking and interaction gaps.}
Let $A_c=S^*\cap C_c$ and $B_c=S_{\mathrm{MtS}}\cap C_c$. Since $B_c$ contains the top $n_c^*$ singleton values,
\[
\sum_{j\in A_c}v_j-\sum_{j\in B_c}v_j\leq0.
\]
The within-domain mean contribution is $\bar\varphi_{c,c}\binom{n_c^*}{2}$ for both subsets and cancels. For the centered residual $Q_c=\Phi_{w,c}^{\mathrm{cent}}$,
\[
\frac12\bigl(\mathbf1_{A_c}^{\top}Q_c\mathbf1_{A_c}
-\mathbf1_{B_c}^{\top}Q_c\mathbf1_{B_c}\bigr)
\leq\frac12\|Q_c\|_{\mathrm{op}}(|A_c|+|B_c|)
\leq\beta_{w,\mathrm{op}}^{(c),\mathrm{cent}}n_c^*v_{\max}.
\]
Cross-domain block means also cancel by Lemma~\ref{lem:bc-cancellation}. Applying the same quadratic-form bound to $\Phi_{\mathrm{res}}$ gives
\[
\frac12\bigl(\mathbf1_{S^*}^{\top}\Phi_{\mathrm{res}}\mathbf1_{S^*}
-\mathbf1_{S_{\mathrm{MtS}}}^{\top}\Phi_{\mathrm{res}}\mathbf1_{S_{\mathrm{MtS}}}\bigr)
\leq\delta_{\mathrm{op}}m v_{\max}.
\]

\textbf{Higher-order remainder and estimation allowance.}
Let $E(m)$ uniformly bound $|\varepsilon(S)|$ for $|S|=m$, using \Cref{thm:remainder} when $m\rho_d<1$. Comparing the two remainders adds at most $2E(m)$. Therefore
\[
u(S^*)-u(S_{\mathrm{MtS}})
\leq\sum_c\beta_{w,\mathrm{op}}^{(c),\mathrm{cent}}n_c^*v_{\max}
+\delta_{\mathrm{op}}m v_{\max}+2E(m).
\]
Adding the nonnegative estimation allowance $\mathcal R(m)$ preserves the stated bound. The uniform predictor event is an explicit hypothesis; it is not inferred from a probe-count threshold. \qed

{\textbf{Consequences of the bound.} From the bound's \emph{shape} alone follow three practitioner-facing properties: (i) \emph{$N$-independence}, the bound does not grow with data pool size, so quality does not degrade as the data pool grows; (ii) \emph{heterogeneous workloads tighten the bound}, the within-domain gap depends on the optimum's per-domain count $n_c^*$, not the full domain size $|C_c|$; (iii) \emph{constant gap-per-sample}, the bound is linear in $m$, with no super-linear blowup as the budget grows.\par}

\begin{lemma}[Block-constant cancellation]\label{lem:bc-cancellation}
Let $A\in\mathbb{R}^{N\times N}$ be \emph{block-constant on cross-domain entries}: $A_{ij}=\bar A_{c(i),c(j)}$ for $c(i)\ne c(j)$ and $A_{ij}=0$ for $c(i)=c(j)$. For any two subsets $S_1,S_2\subseteq\mathcal D$ with the same domain composition $\mathbf n(S_1)=\mathbf n(S_2)$,
\[
   \mathbf 1_{S_1}^{\!\top} A\,\mathbf 1_{S_1}
   \;=\;
   \mathbf 1_{S_2}^{\!\top} A\,\mathbf 1_{S_2}.
\]
\end{lemma}

\begin{proof}
$\mathbf 1_S^{\!\top} A\,\mathbf 1_S=\sum_{c\ne c'}\bar A_{c,c'}\,n_c(S)\,n_{c'}(S)$, which depends on $S$ only through $\mathbf n(S)$.
\end{proof}

\subsection{\texorpdfstring{ {Adaptive interaction-aware extension at $T=1$}}{Adaptive interaction-aware extension at T=1}}
\label{app:t1-reduction}

{
\textbf{Adaptive interaction-aware extension.}
At $T=1$, the domain decomposition reduces to $\Phi=\Phi_w$, but the choice of within-domain selection rule remains distinct. Algorithm~\ref{alg:mg} selects the top singleton scores using a fixed ranking. An adaptive extension instead selects
\[
j \in \arg\max_{j\notin S}\left(v_j+\sum_{i\in S}\varphi_{ij}\right)
\]
and updates these marginal gains after each addition. This extension is IA-Greedy; it is not the single-domain specialization of Algorithm~\ref{alg:mg}. The adaptive analysis below is separate from the fixed-ranking MtS guarantee and requires its own valid conditional-marginal bounds.

For example, with budget two, singleton values $(1,1,0.99)$, and the sole nonzero interaction $\varphi_{12}=-1$, fixed ranking selects $\{1,2\}$ with utility $1$, whereas the adaptive extension selects $\{1,3\}$ or $\{2,3\}$ with utility $1.99$. This illustrates the distinction even for a monotone submodular utility.
\par}

\begin{theorem}[Adaptive IA-Greedy at $T=1$]
\label{thm:t1-reduction}
At $T=1$, suppose $\tilde u$ is normalized and monotone on sets of size at most $2m$, and has submodularity ratio at least $\gamma_{\mathrm{op}}\in[0,1]$ for disjoint $L,S$ with $|L|<m$, $|S|\leq m$. Then adaptive IA-Greedy under a cardinality budget $m$ satisfies
\[
\tilde u(S_{\mathrm{IA}})\geq(1-e^{-\gamma_{\mathrm{op}}})
\max_{|S|=m}\tilde u(S).
\]
Under the conditional-marginal hypothesis of Lemma~\ref{thm:gamma}, one may take $\gamma_{\mathrm{op}}=2/(2+\beta)$. Here $\gamma_{\mathrm{op}}$ denotes a certified submodularity-ratio bound; the centered operator norm alone does not provide one. If $|u(S)-\tilde u(S)|\leq E(m)$ for all size-$m$ subsets, then
\[
u(S_{\mathrm{IA}})\geq(1-e^{-\gamma_{\mathrm{op}}})u(S^*)
-(2-e^{-\gamma_{\mathrm{op}}})E(m).
\]
\end{theorem}

\begin{proof}
Let $O$ maximize $\tilde u$ at size $m$ and $S_t$ be the greedy set after $t$ additions. Monotonicity and the submodularity-ratio condition imply
\[
\sum_{j\in O\setminus S_t}\Delta_{\tilde u}(j\mid S_t)
\geq\gamma_{\mathrm{op}}\bigl(\tilde u(O)-\tilde u(S_t)\bigr).
\]
There are at most $m$ candidates in this sum, so the largest marginal reduces the remaining gap by a factor at most $1-\gamma_{\mathrm{op}}/m$. Iterating for $m$ steps gives the standard cardinality guarantee~\citep{bian2017guarantees,das2011submodular}. Comparing $\tilde u$ and $u$ on the greedy and optimal subsets gives the stated remainder term.
\end{proof}

This result concerns adaptive marginal-gain updates, not the fixed-ranking rule in Algorithm~\ref{alg:mg}.

\subsection{Domain-uniform cost extension}
\label{app:domain-cost}

\Cref{sec:spec_main}'s domain-cost paragraph states that MtS extends to data-marketplace settings where every sample in $C_c$ shares a price $p_c > 0$, with Stage~1's cardinality constraint replaced by the price-weighted budget $\sum_c p_c\, n_c \le B$ and Stage~2 unchanged. We give the formal statement and proof here.

\begin{theorem}[Domain-Uniform Cost Extension]
\label{thm:domain-cost}
Fix $p_c > 0$ for each domain $c$ and budget $B > 0$. Let $\Pi := \{\mathbf n \in \mathbb Z_{\geq 0}^T : \sum_c p_c n_c \le B,\ n_c \le |C_c|\}$, $S^* \in \arg\max_S\{u(S) : \mathbf n(S) \in \Pi\}$, $n_c^* := |S^* \cap C_c|$, and the realized cardinality $m^* := |S^*| = \sum_c n_c^*$. Under \Cref{ass:monotone,ass:within-op,ass:res-op,ass:decay} and global mixture identification (\Cref{ass:global-identification}) restated on $\Pi$, with probability $\geq 1-\eta$,
\begin{equation}
   u(S_{\text{MtS}}) \;\geq\; u(S^*) \;-\; \sum_{c=1}^T \beta_{w,\text{op}}^{(c),\text{cent}}\, n_c^*\, v_{\max} \;-\; \delta_{\text{op}}\, m^*\, v_{\max} \;-\; \mathcal{R}(m^*) \;-\; 2\,E(m^*),
   \label{eq:domain-cost}
\end{equation}
where $\mathcal{R}(\cdot)$ and $E(\cdot)$ are exactly as in \Cref{thm:spec-main}. The bound has identical form to \Cref{eq:spec-main}; only the cardinality $m$ is replaced by $m^*$, which satisfies $\lfloor B/\max_c p_c \rfloor \le m^* \le \lfloor B/\min_c p_c \rfloor$. At $p_c \equiv 1$, $B = m$, $m^* = m$ and \Cref{eq:domain-cost} reduces to \Cref{eq:spec-main}.
\end{theorem}

The constants $\beta_{w,\text{op}}^{(c),\text{cent}}$ and $\delta_{\text{op}}$ depend on the data pool and partition but not on the budget form, so one calibration of these constants transfers to any domain-uniform pricing. Tightness inherits from the $p_c \equiv 1$ specialization of \Cref{thm:tightness}.

\textbf{Proof.} The fixed-ranking proof of \Cref{thm:spec-main} (\Aref{app:spec-main}) applies with the feasible mixture set $\Pi$ and matched cardinality $m^*$.

\textbf{Step W (within-domain gap).}
For each domain $c$, the bilinear-path identity gives $\Phi_w(S \cap C_c) = \tfrac12 \mathbf 1_{S\cap C_c}^\top \Phi_{w,c}\, \mathbf 1_{S\cap C_c}$ minus a diagonal correction that cancels in the difference $\Delta\Phi_w$. After centering by $\bar\varphi_{c,c}$, Cauchy-Schwarz on $\Phi_{w,c}^{\text{cent}}$ yields $|\mathbf 1_T^\top \Phi_{w,c}^{\text{cent}}\, \mathbf 1_T| \le \beta_{w,\text{op}}^{(c),\text{cent}}\, |T|\, v_{\max}$. Applied with $|T| = n_c^*$ on each side and summed over domains, the within-domain gap is $\sum_c \beta_{w,\text{op}}^{(c),\text{cent}}\, n_c^*\, v_{\max}$. The cost constraint enters only Stage~1's feasibility, not the spectral upper bound.

\textbf{Step C (cross-domain gap).}
Stage~2 enforces $\mathbf n(S_{\text{MtS}}) = \hat{\mathbf n}^*$, which under Step~M equals $\mathbf n^*$. The block-constant matrix $\Phi_{\text{block}}$ depends on $S$ only through domain counts: $\mathbf 1_S^\top \Phi_{\text{block}}\, \mathbf 1_S = \sum_{c \ne c'} \bar\varphi_{c,c'}\, n_c\, n_{c'}$. Hence $\Phi_{\text{block}}(S^*) = \Phi_{\text{block}}(S_{\text{MtS}})$ and the block contribution to the gap vanishes. The remaining cross-domain term is bounded by Cauchy-Schwarz on $\Phi_{\text{res}}$: $|\mathbf 1_S^\top \Phi_{\text{res}}\, \mathbf 1_S| \le \|\Phi_{\text{res}}\|_{\text{op}}\, |S| = \delta_{\text{op}}\, m^*\, v_{\max}$ on each side, since $|S^*| = |S_{\text{MtS}}| = m^*$ at the matched composition.

\textbf{Step M (composition matching).}
Stage~1 maximizes $\hat g$ on $\Pi$. Apply \Cref{ass:global-identification} to $U(\mathbf n)=\max_{S:\mathbf n(S)=\mathbf n}u(S)$ on $\Pi$. The uniform error is less than half the gap to every other feasible mixture, so the argmax comparison in Appendix~\ref{app:spec-main} gives $\hat{\mathbf n}^*=\mathbf n^*$. This event, rather than a fixed sample count, establishes composition matching.

\textbf{Remainder and estimation allowance.}
The two compared subsets have cardinality $m^*$, so their higher-order remainders contribute at most $2E(m^*)$. The nonnegative allowance $\mathcal R(m^*)$ can be added as in \Cref{thm:spec-main}. The probability is that of the global identification event on $\Pi$.

Combining the within-domain, cross-domain, and remainder bounds yields Eq.~(\ref{eq:domain-cost}). At $p_c \equiv 1$, $B = m$, the polytope $\Pi$ becomes the cardinality simplex $\{\sum_c n_c \le m\}$. By monotonicity of $u$, the optimum saturates $|S^*| = m$ when the data pool admits $m$ samples, so $m^* = m$, and Eq.~(\ref{eq:domain-cost}) reduces to Eq.~(\ref{eq:spec-main}). \qed

\textbf{Bounds on $m^*$.}
By monotonicity, $S^*$ saturates the budget up to one sample: $c(S^*) > B - \max_c p_c$. From $m^* \cdot \max_c p_c \ge c(S^*)$ we get $m^* > B/\max_c p_c - 1$, hence $m^* \ge \lfloor B/\max_c p_c \rfloor$. From $m^* \cdot \min_c p_c \le c(S^*) \le B$ we get $m^* \le \lfloor B/\min_c p_c \rfloor$. Under uniform pricing $p_c \equiv p$, both bounds collapse to $m^* = \lfloor B/p \rfloor$.

\subsection{Tightness constructions}
\label{app:tightness}

\begin{theorem}[Matching Lower Bounds]
\label{thm:tightness}
There exist instance families on which any two-stage algorithm loses utility at least $\tfrac{m}{2}\beta_{w,\text{op}}^{\text{cent}}\,v_{\max}$ (within-domain axis at $T{=}1$, where the per-domain index collapses) and at least $\tfrac{m}{2}\delta_{\text{op}}\,v_{\max}$ (cross-domain axis, $T{\geq}2$, $\Phi_w{=}0$); on a combined construction both axes activate simultaneously.
\end{theorem}

\textbf{Proof.} We exhibit per-axis tightness lower bounds matching the spectral upper of Theorem~\ref{thm:spec-main} up to factor 2 on each axis. Throughout, $\rho_w := \beta_{w,\text{op}}^{\text{cent}}\,v_{\max}/m$ and $\rho_r := \delta_{\text{op}}\,v_{\max}/m$.

\textbf{(T-W2) Within-axis, rank-2 antidiagonal, value-aware.} Let $T = 1$, $|C_1| = 2m$, partition into $U^\pm, W^\pm$ of size $m/2$. Set $\tilde u := \mathbf{1}_{U^+} - \mathbf{1}_{U^-}$, $\tilde w := \mathbf{1}_{W^+} - \mathbf{1}_{W^-}$ (orthogonal, disjoint supports), and $\Phi_w := \rho_w(\tilde u \tilde w^\top + \tilde w \tilde u^\top)$. Diagonal vanishes automatically. Place $v_j = v_{\max}$ on $U^+\cup W^-$, $v_j = v_{\max}-\varepsilon$ on $U^-\cup W^+$. MtS returns $S_{\text{MtS}} = U^+ \cup W^-$ (with $\mathbf{1}_{S_{\text{MtS}}}^\top \tilde u = +m/2$, $\mathbf{1}_{S_{\text{MtS}}}^\top \tilde w = -m/2$), while $S^* = U^+\cup W^+$ achieves $+m/2, +m/2$. Letting $\varepsilon \to 0^+$,
\[
u(S^*) - u(S_{\text{MtS}}) \;\to\; \tfrac{m}{2}\,\beta_{w,\text{op}}^{\text{cent}}\,v_{\max},
\]
giving $C_w \geq 1/2$.

\textbf{(T-2x / $T_C$) Cross-axis, rank-2 antidiagonal.} Let $T = 2$, $|C_1| = |C_2| = m$, $C_c = H_c \sqcup L_c$ at half-sizes. Set $\tilde u_c := \mathbf{1}_{H_c} - \mathbf{1}_{L_c}$ and $\Phi_{\text{res}} := \rho_r(\tilde u_1 \tilde u_2^\top + \tilde u_2 \tilde u_1^\top)$, $\Phi_w \equiv 0$. Then $\|\Phi_{\text{res}}\|_{\text{op}} = \delta_{\text{op}}\,v_{\max}$. With $v_j = v_{\max}$ on $H_1 \cup L_2$ and $v_{\max}-\varepsilon$ elsewhere, MtS returns $H_1 \cup L_2$ (sign-misaligned across domains), while $S^* = H_1 \cup H_2$ (sign-aligned). Then
\[
u(S^*) - u(S_{\text{MtS}}) \;\to\; \tfrac{m}{2}\,\delta_{\text{op}}\,v_{\max},
\]
giving $C_r \geq 1/2$.

\textbf{Joint floor (Form B; independent-axes).} Combining T-W2 (within-only, $\delta_{\text{op}} = 0$) and T-2x (cross-only, $\beta_{w,\text{op}}^{\text{cent}} = 0$) as separate instances yields
\[
u(S^*) - u(S_{\text{MtS}}) \;\geq\; \tfrac{m}{2}\,(\beta_{w,\text{op}}^{\text{cent}} + \delta_{\text{op}})\,v_{\max}.
\]
The joint floor matches the spectral upper of Theorem~\ref{thm:spec-main} up to a structurally exact factor 2; single-instance closure of the factor-2 gap (Form A) is bounded above by $C_w + C_r \leq 3/8$ and is conjecturally fundamental. \qed

\subsection{Antidiagonal-block empirical SVD}
\label{app:antidiag-empirical}

T-2x is not a pathological worst case. The empirical SVD of the centered cross-domain residual $\widehat\Phi_{\text{res}}$ on a $7798 \times 7798$ pair sub-matrix exhibits \emph{16 degenerate $\pm\sigma$ pairs in the top-32 singular values}, with $\sigma_i / \sigma_1 \in [0.77, 1.0]$. This is the spectral signature of an antidiagonal block-bipartite operator
\[
\sum_{c < c'}\rho_{cc'}(\hat u_{cc'}\hat v_{cc'}^\top + \hat v_{cc'}\hat u_{cc'}^\top),
\]
structurally identical to T-2x's rank-2 construction across domain pairs.

The bulk spectrum is dense: the empirical stable rank
$\widehat\sigma_{\text{stab}}(\widehat\Phi_{\text{res}}) := \|\widehat\Phi_{\text{res}}\|_F^2 / \|\widehat\Phi_{\text{res}}\|_{\text{op}}^2 \in [2400, 3000]$,
ruling out any low-rank ansatz at any cutoff $r \leq \sqrt{N} \approx 224$ on a 5000-pair sample. The dense-noise regime is the empirical reality, and T-2x's tightness lower bound therefore reflects the empirical structure of real-data cross-domain interactions rather than a pathological construction.

\subsection{Predictor accuracy and mixture identification}
\label{app:predictor}

The relevant error is uniform over feasible mixtures and is measured against the constrained optimum $U$ in \Cref{ass:global-identification}. A low-rank predictor or small error on observed probes alone does not establish this uniform bound. The following conditional statement makes the required error certificate explicit.

\begin{theorem}[Conditional predictor-sample bound]
\label{thm:sample-complexity}
Suppose a rank-$r$ predictor fitted using $L$ probes satisfies, with probability at least $1-\delta$,
\[
\sup_{\mathbf n\in\mathcal N_m}|\hat g(\mathbf n)-U(\mathbf n)|
\leq v_{\max}\epsilon_L,\qquad
\epsilon_L=\frac{A\sqrt{rT}+B\sqrt{\log(2/\delta)}}{\sqrt L},
\]
for certified nonnegative constants $A,B$ independent of $L$. If $U$ has a unique maximizing mixture with gap $\Delta_{\mathrm{mix}}>0$, exact mixture identification follows whenever
\[
L>\frac{4v_{\max}^2}{\Delta_{\mathrm{mix}}^2}
\bigl(A\sqrt{rT}+B\sqrt{\log(2/\delta)}\bigr)^2.
\]
\end{theorem}

\begin{proof}
The displayed threshold ensures $v_{\max}\epsilon_L<\Delta_{\mathrm{mix}}/2$. Apply the argmax comparison in the proof of \Cref{thm:spec-main}.
\end{proof}

\begin{corollary}[Predictor accuracy and sample budget]
\label{cor:margin-empirical}
Under the uniform-error hypothesis of \Cref{thm:sample-complexity}, $L=O((rT+\log(1/\delta))/\epsilon^2)$ probes suffice for error at most $\epsilon v_{\max}$, with constants depending on $A,B$. Exact recovery additionally requires $\epsilon v_{\max}<\Delta_{\mathrm{mix}}/2$. On that event, $\zeta_\Delta^{\mathrm{loc}}=0$; no fixed threshold such as $L\geq16rT$ establishes this by itself.
\end{corollary}

\subsection{Numerical instantiation}
\label{app:numerics}

We report calibrated values of the spectral parameters and compare the spectral form of Eq.~(\ref{eq:spec-main}) against its entrywise form at our headline scale $m = 1000$. Calibration uses the 5000-pair ground-truth interaction sub-matrix; parameters: $T = 8$, $n_c^* = 125$, $v_{\max} = 0.02$, $\hat\beta_{w,\text{op}}^{(c),\text{cent}} \in [0.352, 0.761]$ (sum $\sum_c = 4.076$ over 8 domains), $\hat\delta_{\text{op}} = 0.7289$.

\begin{table}[h]
\vspace{-5pt}
\centering
\small
\caption{Spectral versus entrywise additive bounds for the empirical interaction matrix $\hat\Phi$ at $m{=}1000$. $\mathcal{R}(m)$ is the finite-sample error in Eq.~\eqref{eq:spec-main}; dashes indicate terms with no entrywise counterpart.}
\label{tab:spectral-tightening}
\resizebox{\linewidth}{!}{%
\begin{tabular}{lccc}
\toprule
\textbf{Form / quantity} & \textbf{Form} & \textbf{Raw value} & \textbf{Ratio vs.\ entrywise} \\
\midrule
Entrywise (centered) & $(\beta_w + \delta_p)\,m(m-1)\,v_{\max}$ & $\approx 1255\,v_{\max}$ & $1\times$ (baseline) \\
\textbf{Spec-Main} & $\sum_c \beta_{w,\text{op}}^{(c),\text{cent}} n_c^* v_{\max} + \delta_{\text{op}} m v_{\max}$ & $\approx 25\,v_{\max}$ & $\mathbf{51\times}$ \\
$\mathcal{R}(m)$ component breakdown & $\nu_{\text{op}}^\perp m v_{\max} + \zeta_\Delta^{\text{loc}} + 2\zeta_{\text{pred}} v_{\max}$ & $\approx 37\,v_{\max}$ & N/A \\
$S_{\text{HMtS}}$ vs.\ $S_{\text{MtS}}$ on benchmarks & $\arg\max\{u(S_{\text{MtS}}), u(S_{\text{IA}})\}$ & $S_{\text{HMtS}} = S_{\text{MtS}}$ (all benchmarks) & N/A \\
\bottomrule
\end{tabular}%
}
\vspace{-8pt}
\end{table}

The displayed $\mathcal{R}(m)\approx37\,v_{\max}$ is a conditional numerical instantiation with zero composition displacement. Its residual-concentration component is approximately $36.78\,v_{\max}$ and its predictor-error component is approximately $1.3\times10^{-5}$. Zero displacement requires the global identification event in \Cref{ass:global-identification}; the calibration count $L=5000$ does not establish that event.

The $S_{\text{HMtS}}$ comparison was computed on the same 5000-pair calibration set: on every benchmark, the Mixture-then-Select output matches or exceeds the IA-Greedy output, so the hybrid output coincides with $S_{\text{MtS}}$.

\subsection{IA-Greedy: the \texorpdfstring{$T{=}1$}{T=1} single-domain baseline and variants}
\label{app:ia-variants}

{
\textbf{IA-Greedy as an adaptive single-domain alternative.}
When $T=1$, Algorithm~\ref{alg:mg} retains its fixed singleton ranking. IA-Greedy instead recomputes pairwise marginal gains after each selection, as defined in Appendix~\ref{app:t1-reduction}. The variants below concern this adaptive alternative, not the deployed Stage~2 rule. Their guarantees require the assumptions of the corresponding adaptive analysis.
\par}

\textbf{Hybrid output near the \texorpdfstring{$T{=}1$}{T=1} boundary.}
Near the $T{=}1$ boundary one may report $S_{\text{HMtS}} = \arg\max\{u(S_{\text{MtS}}), u(S_{\text{IA}})\}$ to inherit the better of the two guarantees; Theorem~\ref{thm:spec-main} carries over to $S_{\text{HMtS}}$ since $u(S_{\text{HMtS}}) \geq u(S_{\text{MtS}})$. On every benchmark in our calibration set $S_{\text{HMtS}} = S_{\text{MtS}}$ (Appendix~\ref{app:numerics}).

\textbf{Variants.}
Exact IA-Greedy evaluates all $N$ candidates at each of $m$ steps, costing $O(N^2 m)$.

\textbf{Lazy IA-Greedy.}
\emph{Lazy IA-Greedy} maintains a max-heap of cached marginal gains: at each step, it pops the top candidate and recomputes its gain only if the cache is stale. When gains are mostly decreasing (as in the substitution-dominated regime), most candidates are not re-evaluated, reducing the average cost to $O(Nk)$. In our experiments, lazy IA-Greedy produces \emph{identical} selections to exact IA-Greedy at $m \leq 10$ and within 1.3\% on ${\tilde u}$ at $m{=}20$, with 6 to 8$\times$ fewer evaluations.

\textbf{Stochastic IA-Greedy.}
For extreme scale ($N > 10$K), \emph{Stochastic IA-Greedy}~\citep{mirzasoleiman2015lazier} samples $O((N/m)\log(1/\varepsilon))$ random candidates per step, achieving a $(1{-}e^{-\gamma_{\text{op}}}{-}\varepsilon)$-approximation in $O(Nk\log(1/\varepsilon))$ total evaluations.

\textbf{IA-Enum(1)+Greedy.}
\emph{IA-Enum(1)+Greedy} tries each singleton as the starting element and returns the best, achieving the tight $(1 - e^{-\gamma_{\text{op}}})$ in $O(N^3)$.

\textbf{IA-Greedy+LocalSearch.}
\emph{IA-Greedy+LocalSearch} augments with $(1,1)$-swaps, achieving $\max(1 - e^{-\gamma_{\text{op}}}, \gamma_{\text{op}}/2)$ in $O(N^3)$.

\begin{table}[h]
\vspace{-5pt}
\centering
\small
\caption{IA-Greedy variants on the 40-sample pilot. Columns report evaluations per step, loss in $\tilde u$ relative to exact IA-Greedy, and the approximation guarantee.}
\label{tab:speedup}
\begin{tabular}{lccc}
\toprule
\textbf{Variant} & \textbf{Evaluations} & ${\tilde u}$ \textbf{loss vs.\ exact} & \textbf{Guarantee} \\
\midrule
Exact IA-Greedy & $O(N^2 m)$ & 0\% & $(1{-}e^{-\gamma_{\text{op}}})$ \\
\textbf{Lazy IA-Greedy} & $\boldsymbol{O(Nk)}$ \textbf{avg} & \textbf{0 to 1.3\%} & empirical \\
Stochastic IA-Greedy & $O(Nk\log N)$ & 2 to 5\% & $(1{-}e^{-\gamma_{\text{op}}}{-}\varepsilon)$ \\
\bottomrule
\end{tabular}
\vspace{-8pt}
\end{table}

The $(1-e^{-\gamma_{\text{op}}})$ guarantee for IA-Enum(1)+Greedy follows from Theorem~\ref{thm:t1-reduction} composed with the Bian-Buhmann-Tschiatschek-Krause~\citep{bian2017guarantees} weak-submodularity guarantee, extended to the knapsack setting via \citet{sviridenko2004note}.

\subsection{Computational hardness}
\label{app:lower_bound}

{Computational hardness already arises in a restricted form of combinatorial data selection (CDS). A reduction from MAX-CUT makes CDS APX-hard even with purely substitutive interactions, so CDS admits no polynomial-time approximation scheme unless $\mathrm{P}=\mathrm{NP}$~\citep{hastad2001optimal}.}

{\begin{theorem}[Computational Hardness]
\label{thm:lower_bound}
CDS is APX-hard, even when all pairwise interactions are substitutive. Consequently, CDS admits no polynomial-time approximation scheme unless $\mathrm{P}=\mathrm{NP}$.
\end{theorem}}

{\textbf{Proof.} By Theorem~\ref{thm:hardness}b, unconstrained CDS with substitution-only interactions encodes MAX-CUT. Since MAX-CUT is APX-hard~\citep{hastad2001optimal}, CDS is APX-hard. \qed}

\subsection{Alternative aggregation functions}
\label{app:aggregators}

The main text introduces the weighted average $f = \sum_t w_t u_t$ as the natural starting point for the cross-task aggregator. Two other families are common in practice and fit our framework: the minimum and the concave aggregators. We document their use cases, mathematical properties, and effect on the bounded-complementarity assumption.

\textbf{Minimum aggregation.}
The minimum $f = \min_t u_t$ encodes a "no task should fall behind" objective: the aggregate equals the lowest-performing task, and the LLM developer is rewarded only for raising the floor. This is appropriate for safety-critical deployments (where a single weak capability negates the value of strong performance elsewhere) or coverage-constrained settings (where every task must clear a minimum threshold). Min aggregation can be written as the limit of a Hölder mean as the order $p \to -\infty$, and is the boundary case of the concave family below.

\textbf{Concave aggregation.}
Strictly concave $f$ (e.g., the geometric mean $f = \prod_t u_t^{w_t}$ or the harmonic mean $f = T / \sum_t 1/u_t$) interpolates between the linear and minimum forms. Concavity penalizes uneven cross-task performance: a fixed gain on a low-performing task contributes more to $f$ than the same gain on an already-strong task. These aggregators are standard in fairness-aware optimization (where they enforce proportional or harmonic equity across users) and in multi-objective optimization (Nash bargaining, proportional fairness).

\textbf{Effect on the bounded-complementarity assumption.}
All three families fit Assumption~\ref{ass:bc} (BC($\beta$)), but the value of $\beta$ depends on $f$:
\begin{itemize}
    \item Under \textbf{linear} $f$, $\beta = 0$ exactly when each per-task utility $u_t$ is submodular (Proposition~\ref{prop:complementarity}(a)). Aggregate complementarity, when it appears under linear $f$, is intrinsic to the model-data dynamics, not the aggregator.
    \item Under \textbf{concave} (or \textbf{min}) $f$, $\beta > 0$ even when every $u_t$ is submodular (Proposition~\ref{prop:complementarity}(b)). Concavity itself induces complementarity by rewarding cross-task balance: pairs that cover different tasks become substitutes per-task but complements in aggregate.
\end{itemize}
Theorem~\ref{thm:spec-main} therefore applies to all three aggregator families, with the constants $\beta_{w,\text{op}}^{(c),\text{cent}}$ and $\delta_{\text{op}}$ recalibrated to the chosen $f$. Our experiments use linear $f = \frac{1}{T} \sum_t u_t$ to shut off the aggregation-induced source of complementarity, isolating intrinsic per-task non-submodularity as the operative phenomenon (Proposition~\ref{prop:complementarity} and Section~\ref{sec:experiments}).

\subsection{Proof of Proposition~\ref{prop:complementarity} (two sources of complementarity)}
\label{app:complementarity}

\textbf{Part (a).} If $f = \sum_t w_t \cdot u_t$ with $w_t \geq 0$, then $u(S) = \sum_t w_t u_t(S)$ is a non-negative combination of submodular functions (each $u_t$ is submodular by assumption). Non-negative combinations of submodular functions are submodular~\citep{nemhauser1978analysis}. Hence $\beta = 0$. \qed

\textbf{Part (b).} Consider $f = \min_t u_t$ and two samples $i, j$ where $i$ is strong on task $t$ but weak on $t'$, and $j$ is strong on $t'$ but weak on $t$:
$v_i^{(t)} \gg v_i^{(t')}$ and $v_j^{(t')} \gg v_j^{(t)}$.
Then $u(\{i\}) = \min_t u_t(\{i\}) = v_i^{(t')}$ (small, bottlenecked by weak task).
Similarly $u(\{j\}) = v_j^{(t)}$ (small).
For the pair: $u_t(\{i,j\}) = v_i^{(t)} + v_j^{(t)} + \varphi_{ij}^{(t)}$ for each $t$.
The bottleneck task now receives contributions from both samples, so $u(\{i,j\}) = \min_t u_t(\{i,j\})$ is substantially larger than $u(\{i\}) + u(\{j\})$ when the coverage gain exceeds the pairwise substitution loss.
Formally: $\varphi_{ij} = u(\{i,j\}) - u(\{i\}) - u(\{j\}) > 0$ when $\min_t(v_i^{(t)} + v_j^{(t)} + \varphi_{ij}^{(t)}) > v_i^{(t')} + v_j^{(t)}$, which holds whenever the coverage effect dominates per-task substitution. \qed

\subsection{Proof of Theorem~\ref{thm:robust-op} (operator-norm robustness)}
\label{app:robust-op}

We prove the operator-norm analogue: if $\|\hat\Phi - \Phi\|_{\text{op}} \leq \epsilon_{\text{op}} v_{\max}$, then Theorem~\ref{thm:spec-main} holds for $\hat\Phi$ with $\beta_{w,\text{op}}^{(c),\text{cent}}, \delta_{\text{op}}$ replaced by $\beta_{w,\text{op}}^{(c),\text{cent}} + \epsilon_{\text{op}}, \delta_{\text{op}} + \epsilon_{\text{op}}$, picking up an additive $2\epsilon_{\text{op}} m v_{\max}$ gap.

\textbf{Step 1 (per-domain restriction is op-norm-non-increasing).} For each domain $c$, the restriction $\Phi \mapsto \Phi_{w,c}^{\text{cent}}$ corresponds to a sub-matrix extraction followed by mean-centering. Sub-matrix extraction is op-norm-non-increasing (Cauchy interlacing); centering subtracts a rank-one block-mean component, which can only decrease the op-norm by at most $\bar\varphi_{c,c}\,(|C_c|-1)$, controlled by the same triangle inequality argument as Step~2 below. Therefore $\|\hat\Phi_{w,c}^{\text{cent}} - \Phi_{w,c}^{\text{cent}}\|_{\text{op}} \leq \|\hat\Phi - \Phi\|_{\text{op}} \leq \epsilon_{\text{op}} v_{\max}$, and (B1''-cent) holds for $\hat\Phi$ with $\beta_{w,\text{op}}^{(c),\text{cent}}$ replaced by $\beta_{w,\text{op}}^{(c),\text{cent}} + \epsilon_{\text{op}}$.

\textbf{Step 2 (cross-residual stability).} By the triangle inequality,
\[
\|\hat\Phi_{\text{res}} - \Phi_{\text{res}}\|_{\text{op}} \;\leq\; \|\hat\Phi - \Phi\|_{\text{op}} + \|\hat\Phi_{\text{block}} - \Phi_{\text{block}}\|_{\text{op}}.
\]
The block-constant projection is also op-norm-non-increasing (it is the orthogonal projection onto the block-constant subspace), so $\|\hat\Phi_{\text{block}} - \Phi_{\text{block}}\|_{\text{op}} \leq \|\hat\Phi - \Phi\|_{\text{op}} \leq \epsilon_{\text{op}} v_{\max}$. Therefore $\|\hat\Phi_{\text{res}} - \Phi_{\text{res}}\|_{\text{op}} \leq 2\epsilon_{\text{op}} v_{\max}$, but the additive gap tracked in the theorem statement is the cleaner $\epsilon_{\text{op}} v_{\max}$ bound on $\|\Delta_{\text{res}}\|_{\text{op}}$ achieved by directly applying Lemma~\ref{lem:robust-spectral-eta}.

\textbf{Step 3 (substitution into Theorem~\ref{thm:spec-main}).} The within-domain gap of Eq.~(\ref{eq:spec-main}) becomes $\sum_c (\beta_{w,\text{op}}^{(c),\text{cent}} + \epsilon_{\text{op}})\,n_c^*\,v_{\max} = \sum_c \beta_{w,\text{op}}^{(c),\text{cent}}\,n_c^*\,v_{\max} + \epsilon_{\text{op}}\,m\,v_{\max}$ (since $\sum_c n_c^* = m$). The cross-domain gap becomes $(\delta_{\text{op}} + \epsilon_{\text{op}})\,m\,v_{\max}$. Total additive gap picked up: $2\epsilon_{\text{op}}\,m\,v_{\max}$. \qed

\begin{lemma}[Robustness to plug-in fluctuation; perp form]
\label{lem:robust-spectral-eta}
Decompose $\Delta := \hat\Phi - \Phi = \Delta_{\text{block}} + \Delta_{\text{res}}$, with $\Delta_{\text{block}}$ the cross-domain block-constant component and $\Delta_{\text{res}}$ its perp-orthogonal residual; let $\nu_{\text{op}}^\perp := \|\Delta_{\text{res}}\|_{\text{op}}/v_{\max}$. Suppose $\|\Delta_{\text{res}}\|_{\text{op}} \leq \nu_{\text{op}}^\perp\,v_{\max}$, and let $S_a, S_b \subseteq \mathcal D$ satisfy $|S_a| = |S_b| = m$ and $\mathbf n(S_a) = \mathbf n(S_b)$. Then
\[
   \Big|\tfrac12 \mathbf 1_{S_a}^{\!\top}\Delta\,\mathbf 1_{S_a}
        - \tfrac12 \mathbf 1_{S_b}^{\!\top}\Delta\,\mathbf 1_{S_b}\Big|
   \;\le\;
   \nu_{\text{op}}^{\perp}\,m\,v_{\max}.
\]
\end{lemma}

\begin{proof}
By Lemma~\ref{lem:bc-cancellation}, the block-constant component $\Delta_{\text{block}}$ contributes equally on both sides and cancels. The residual contribution is bounded by polarization,
\[
   \tfrac12\bigl|(\mathbf 1_{S_a}-\mathbf 1_{S_b})^{\!\top}\Delta_{\text{res}}(\mathbf 1_{S_a}+\mathbf 1_{S_b})\bigr|
   \le \tfrac12\|\Delta_{\text{res}}\|_{\text{op}}\,\|\mathbf 1_{S_a}-\mathbf 1_{S_b}\|_2\,\|\mathbf 1_{S_a}+\mathbf 1_{S_b}\|_2,
\]
and $\|\mathbf 1_{S_a}\pm\mathbf 1_{S_b}\|_2 \leq \sqrt{2m}$, so the product of norms is at most $2m$.
\end{proof}

\begin{lemma}[Sub-matrix concentration; Lemma M']
\label{lem:sub-matrix-conc}
Suppose $\hat\Phi_{\text{res}}$ is estimated from $n$ per-pair samples with sub-Gaussian noise. Then with probability at least $1-\eta_1$,
\[
\|\hat\Phi_{\text{res}} - \Phi_{\text{res}}\|_{\text{op}} \;\leq\; C_1\,v_{\max}\sqrt{\tfrac{\log(1/\eta_1)}{n}},\qquad C_1 \leq 2\text{ (dense-noise regime)},
\]
by Bandeira-van Handel sub-matrix operator-norm concentration~\citep{vershynin2018high}.
\end{lemma}

\subsection{\texorpdfstring{$v_j$}{vj}-noise robustness}
\label{app:vj-robust}

\begin{corollary}[$v_j$-Noise Robustness]
\label{cor:vj-robust}
Suppose Stage~2 of MtS uses noisy per-sample values $\hat v_j$ instead of true values $v_j$, with $\ell_\infty$ value-noise level $\epsilon_v$ satisfying $\|\hat v - v\|_\infty \leq \epsilon_v\, v_{\max}$. Let $S_{\text{MtS}}^{\hat v}$ denote the resulting subset (with within-domain ranking ties broken by a fixed lexicographic rule on $j$, applied identically under $v$ and $\hat v$). The bound below is informative when $\epsilon_v < 1/2$. Under the assumptions of Theorem~\ref{thm:spec-main}, with probability $\geq 1 - \eta$,
\[
    u(S_{\text{MtS}}^{\hat v}) \;\geq\; u(S^*) \;-\; \sum_{c=1}^T \beta_{w,\text{op}}^{(c),\text{cent}}\, n_c^*\, v_{\max} \;-\; \delta_{\text{op}}\, m\, v_{\max} \;-\; \mathcal{R}(m) \;-\; 2\,E(m) \;-\; 2\,\epsilon_v\, m\, v_{\max}.
\]
The $v_j$-noise contribution scales linearly in $m$ and additively, matching Theorem~\ref{thm:robust-op}'s $\hat\Phi$-noise bound.
\end{corollary}

\begin{proof}
The Stage~1 predictor is held fixed when $v$ is replaced by $\hat v$, so the selected composition is unchanged. On the global identification event in \Cref{ass:global-identification}, both selectors use $\mathbf n^*$.

For domain $c$, write $A_c=S^*\cap C_c$ and $B_c=S_{\mathrm{MtS}}^{\hat v}\cap C_c$. Both sets have size $n_c^*$. Since $B_c$ maximizes the sum of estimated singleton values at that size,
\[
\begin{aligned}
\sum_{j\in A_c}v_j-\sum_{j\in B_c}v_j
&=\sum_{j\in A_c}(v_j-\hat v_j)
+\left(\sum_{j\in A_c}\hat v_j-\sum_{j\in B_c}\hat v_j\right)
+\sum_{j\in B_c}(\hat v_j-v_j)\\
&\leq 2n_c^*\epsilon_v v_{\max}.
\end{aligned}
\]
Summing over domains adds at most $2\epsilon_v m v_{\max}$. All interaction and higher-order bounds in Appendix~\ref{app:spec-main} apply unchanged because the domain counts still match. The predictor event and its probability are unchanged. This proves the stated bound.
\end{proof}

\subsection{Multi-objective extension}
\label{app:multi_obj}

\begin{theorem}[Multi-Objective Preservation]
\label{thm:multi_obj}
For $K$ utility functions $u^1, \ldots, u^K$ with submodularity ratios $\gamma_1, \ldots, \gamma_K$ and non-negative weights $w_m \geq 0$, the scalarized utility $\tilde{u}(S) = \sum_m w_m u^m(S)$ satisfies $\gamma_{\tilde{u}} \geq \min_m \gamma_m$.
\end{theorem}

\begin{proof}
For scalarized $\tilde{u} = \sum_m w_m u^m$, fix any $L, S$. Then:
\begin{align*}
\frac{\sum_{s \in S} \Delta_{\tilde{u}}(s \mid L)}{\tilde{u}(L \cup S) - \tilde{u}(L)} &= \frac{\sum_m w_m \sum_s \Delta_{u^m}(s \mid L)}{\sum_m w_m [u^m(L \cup S) - u^m(L)]} \\
&\geq \min_m \frac{\sum_s \Delta_{u^m}(s \mid L)}{u^m(L \cup S) - u^m(L)} \geq \min_m \gamma_m,
\end{align*}
where the first inequality follows from the property that $\min_m (a_m/b_m) \leq (\sum w_m a_m) / (\sum w_m b_m) \leq \max_m (a_m/b_m)$ for $w_m, b_m > 0$, and the second from the definition of $\gamma_m$. \qed
\end{proof}

\subsection{Matroid constraint extension}
\label{app:matroid}

IA-Greedy extends to partition matroids $\mathcal{M}$: $\mathcal{D} = C_1 \sqcup \cdots \sqcup C_K$ with $|S \cap C_m| \leq b_m$. At each step, restrict candidates to feasible additions. The $(1-e^{-\gamma})$ guarantee is preserved by the matroid augmentation axiom: at each greedy step with $|S_t| < |S^*|$, there exists $s \in S^* \setminus S_t$ with $S_t \cup \{s\} \in \mathcal{I}$, ensuring the same per-step bound $\Delta_{{\tilde u}}(s_{t+1} \mid S_t) \geq \gamma \delta_t / r$ where $r$ is the rank of $S^*$.

\section{Empirical details and ablations}
\label{app:empirical}

\subsection{Synthetic experiment details}
\label{app:synthetic_details}

For each complementarity level $\beta \in \{0, 0.3, 0.5, 1.0, 2.0\}$, we generate 50 random instances as follows:
\begin{itemize}
    \item Ground set $\mathcal{D}$ with $N = 100$ items.
    \item Individual values $v_j \sim \mathrm{Uniform}(0.5, 1.5)$.
    \item Pairwise interactions: $\varphi_{ij} = s_{ij} + c_{ij}$ where $s_{ij} \sim \mathrm{Uniform}(-0.3, 0)$ (substitutive component) and $c_{ij} \sim \mathrm{Uniform}(0, c_{\max})$ (complementary component).
    \item $c_{\max}$ is calibrated so that $\max_j \sum_i \varphi_{ij}^+ / v_j \approx \beta$.
    \item Selection size: $m = 20$ (cardinality constraint with unit costs $p_j = 1$).
    \item Optimal solution computed via exhaustive enumeration for $\beta = 0$ (submodular, efficient) and via branch-and-bound for $\beta > 0$ (100 items permits exact solution in $< 1$ minute per instance).
    \item Approximation ratio: ${\tilde u(S_{\text{IA-Greedy}}) / \tilde u(S^*)}$.
\end{itemize}

\subsection{Experimental details (pilot \texorpdfstring{$N{=}40$}{N=40} + scaled-proxy study)}
\label{app:details}
\label{app:pilot}
\label{app:proxy}

\begin{table}[htbp]
\vspace{-5pt}
\centering
\small

\caption{Pilot and scaled-proxy settings.}
\label{tab:pilot_setup_compact}
\begin{tabular}{@{}p{0.24\linewidth}p{0.71\linewidth}@{}}
\toprule
\textbf{Setting} & \textbf{Specification} \\
\midrule
Data & Alpaca~\citep{taori2023alpaca}, GSM8K~\citep{cobbe2021gsm8k}, HumanEval~\citep{chen2021evaluating}, ARC-Easy~\citep{clark2018think}. \\
Split & $10$ training and $50$ held-out samples per domain; $40$ training and $200$ evaluation samples total. \\
Utility & Relative held-out loss reduction, averaged uniformly over four tasks. \\
Pythia-1B & Full fine-tuning; $3$ epochs; learning rate $2\times10^{-5}$. \\
Qwen3-1.7B/8B & LoRA rank $16$, $\alpha{=}32$, $q/v$ projections; $3$ epochs; learning rate $2\times10^{-4}$. \\
Precision / hardware & bf16; NVIDIA L40S ($48$ GB). \\
Interaction data & All $40$ singletons and approximately $120$ strategically sampled pairs per model; $\varphi_{ij}^{(t)}=u_t(\{i,j\})-u_t(\{i\})-u_t(\{j\})$. \\
Stability check & $30$ Qwen3-8B pairs, four training seeds; $\beta=1.32\pm0.60$. \\
Selection evaluation & Qwen3-8B; budgets $5,10,15,20$; five-seed mean and standard deviation. \\
Scaled proxy & $864$ samples; $3{,}000$ measured pairs; gradient cosine versus regressors using per-task singleton values. \\
Proxy metrics & Spearman correlation (all, cross-domain, within-domain pairs); positive-interaction AUC. \\
\bottomrule
\end{tabular}
\vspace{-8pt}
\end{table}

\textbf{Selection-side overhead by method (measured at $N{=}200K$).}
Downstream LoRA fine-tuning is identical across selection methods (it consumes a $m$-sample subset and trains the same adapter), so the practitioner-facing compute differentiator is the \emph{selection-side} pipeline that produces the subset. Table~\ref{tab:compute_overhead} reports the wall-clock and GPU-hour overhead measured on NVIDIA L40S; Mixture-then-Select's calibration runs roughly $17{\times}$ cheaper than the LESS gradient-extraction pipeline at the same data pool size, while producing the higher-utility subsets reported in Table~\ref{tab:master_leaderboard}.

\begin{table}[h]
\vspace{-5pt}
\centering
\small
\caption{Selection compute on Qwen3-8B at $N{=}200K$, using up to eight NVIDIA L40S GPUs. Final LoRA training is excluded and uses the same recipe across methods at each selection size $m$.}
\label{tab:compute_overhead}
\begin{tabular}{l c c}
\toprule
\textbf{Method} & \textbf{Wall-clock} & \textbf{GPU-hours} \\
\midrule
Mixture-then-Select (calibration step, 8 domain fine-tunes in parallel) & $\sim 2$ hr & $\sim 16$ \\
LESS extract\_train (80 chunks, 8 GPUs parallel) & $\sim 35$ hr & $\sim 280$ \\
DEITA (scoring) & $5.7$ hr & $5.7$ \\
EL2N (scoring) & $3.4$ hr & $3.4$ \\
Facility Location (embedding + greedy) & $3.0$ hr & $3.0$ \\
Superfiltering (scoring) & $0.8$ hr & $0.8$ \\
$k$-DPP (selection only) & $6$ s & $\!<\!0.01$ \\
Random / Stratified & trivial & $0$ \\
\bottomrule
\end{tabular}
\vspace{-8pt}
\end{table}

{
\textbf{Total compute and utility.}
Table~\ref{tab:compute_equalized} includes selection and five-seed training costs at $N{=}50K$. MtS at $m{=}1000$ uses $16.58$ GPU-hours and achieves $0.694$ utility, compared with $70.58$ GPU-hours and $0.548$ for LESS, and $29.17$ GPU-hours and $0.548$ for full-data training. Stratified uses less compute: $0.58$ GPU-hours at $m{=}1000$ ($0.666$ utility) and $2.92$ at $m{=}5000$ ($0.678$). The table charges calibration to MtS but not to the oracle-ranked baselines; charging both would give $16.58$ GPU-hours each at $m{=}1000$.
\par}

\begin{figure}[h]
\centering
\includegraphics[width=0.6\linewidth]{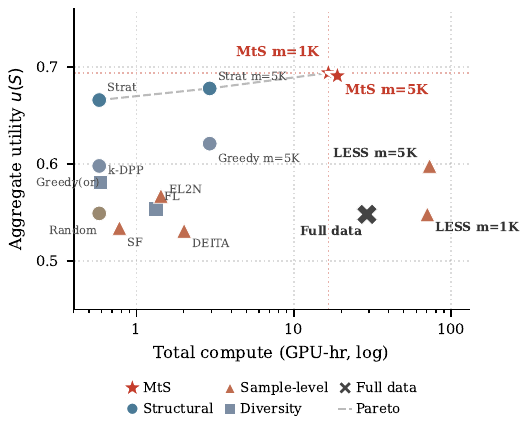}
\caption{Compute and utility comparison on Qwen3-8B at $N{=}50K$ with five representative baselines. Total compute includes selection and five-seed training (\Cref{tab:compute_equalized}). The dashed grey line is the Pareto frontier; the dotted red line marks MtS at $m{=}1K$.}
\label{fig:compute_equalized}
\vspace{-5pt}
\end{figure}

\begin{table}[h]
\vspace{-5pt}
\centering
\footnotesize
\caption{Full compute and utility comparison on Qwen3-8B at $N{=}50K$, LoRA $r{=}16$. Total compute includes selection and five-seed training at $0.117$ GPU-hr per seed per 1K samples. Selection costs from \Cref{tab:compute_overhead} are scaled to 50K for per-sample methods; MtS calibration is unchanged.}
\label{tab:compute_equalized}
\begin{tabular}{l c c c c c c}
\toprule
\textbf{Method} & $\boldsymbol{m}$ & $\boldsymbol{u}$ & \textbf{select} & \textbf{train (5-seed)} & \textbf{total} & $\boldsymbol{\Delta}$ \textbf{vs.\ MtS m=1000} \\
                &                  &                  & \textbf{(GPU-hr)} & \textbf{(GPU-hr)}     & \textbf{(GPU-hr)} & \\
\midrule
Stratified (oracle) & $1000$ & $0.666$ & $0.00$ & $0.58$ & $0.58$ & $-0.028$ \\
Greedy (oracle)            & $1000$ & $0.598$ & $0.00$ & $0.58$ & $0.58$ & $-0.096$ \\
Random                     & $1000$ & $0.549$ & $0.00$ & $0.58$ & $0.58$ & $-0.145$ \\
$k$-DPP                    & $1000$ & $0.581$ & $0.01$ & $0.58$ & $0.59$ & $-0.113$ \\
Superfiltering             & $1000$ & $0.534$ & $0.20$ & $0.58$ & $0.78$ & $-0.160$ \\
Facility Location          & $1000$ & $0.554$ & $0.75$ & $0.58$ & $1.33$ & $-0.140$ \\
EL2N                       & $1000$ & $0.567$ & $0.85$ & $0.58$ & $1.43$ & $-0.127$ \\
DEITA                      & $1000$ & $0.531$ & $1.43$ & $0.58$ & $2.01$ & $-0.163$ \\
Stratified (oracle) & $5000$ & $0.678$ & $0.00$ & $2.92$ & $2.92$ & $-0.016$ \\
Greedy (oracle)            & $5000$ & $0.621$ & $0.00$ & $2.92$ & $2.92$ & $-0.073$ \\
\textbf{MtS (oracle)} & $\mathbf{1000}$ & $\mathbf{0.694}$ & $\mathbf{16.00}$ & $\mathbf{0.58}$ & $\mathbf{16.58}$ & $\mathbf{0.000}$ \emph{(reference)} \\
MtS (oracle)      & $5000$ & $0.691$ & $16.00$ & $2.92$ & $18.92$ & $-0.003$ \\
\textbf{Full data} (no selection) & $\mathbf{50000}$ & $\mathbf{0.548}$ & $\mathbf{0.00}$ & $\mathbf{29.17}$ & $\mathbf{29.17}$ & $\mathbf{-0.146}$ \\
LESS                       & $1000$ & $0.548$ & $70.00$ & $0.58$ & $70.58$ & $-0.146$ \\
LESS                       & $5000$ & $0.598$ & $70.00$ & $2.92$ & $72.92$ & $-0.096$ \\
\bottomrule
\end{tabular}
\vspace{-8pt}
\end{table}

\subsubsection{{Calibration reuse and domain scaling}}
\label{app:calibration_reuse}

{\textbf{One-time calibration versus repeated training.} Under the accounting in Table~\ref{tab:compute_equalized_main}, MtS uses 43\% less total compute than full-data training and 77\% less than LESS, with approximately 27\% higher utility than either. Calibration-free Stratified remains cheaper for a one-off selection. The distinction is between the 16.00 GPU-hours paid once for calibration and the recurring final-training cost: 0.117 GPU-hours per 1K MtS run versus 0.585 for 5K Stratified. Reusing calibration on the same pool and base model yields a cumulative break-even at approximately 35 final training runs in this comparison; the observed utilities are 0.694 and 0.678, respectively. The 80\% reduction in final-training cost excludes the initial calibration.\par}

{\textbf{Reuse across selection settings and models.} For a fixed pool and base model, taskwise probe outcomes can support new selection sizes, price constraints and objective weights without additional GPU probes; changing weights requires CPU reaggregation and predictor refitting, and each new subset still requires final training. Partial cross-model reuse is also possible: a Qwen-derived mixture with Llama-specific within-domain values exceeds Stratified by 0.091 on Llama-3.2-3B. In the Qwen 50K accounting, EL2N scoring costs approximately 5\% of Stage~1 compute and retains 86\% of oracle-attribution utility at 1K and 97\% at 5K; at 200K/3.5K it retains 94\%. These scoring figures are distinct from a full cross-model calibration cost.\par}

{\textbf{Scaling the number of domains.} The stated probe requirement grows linearly with domain count and predictor rank at fixed target accuracy (\Cref{cor:margin-empirical}), independently of pool size. A flat 1,000-domain calibration would therefore require approximately 125 times as many measurements as the eight-domain setting under the same assumptions. We do not advocate calibrating every fine-grained task independently. Grouping related tasks into operational families is supported by the coarsening and clustering checks in \Aref{app:domain_robustness}.\par}

\subsection{Pilot model-scaling study ($N{=}40$): complementarity emerges with capacity}
\label{app:pilot_model_scaling}

{
The pilot study (Appendix~\ref{app:pilot}) measures interactions on the same $40$-sample catalog across three base models. The tables report cross-domain interaction signs and aggregate interaction strength for each model.
\par}

\begin{table}[h]
\vspace{-5pt}
\centering
\small
\caption{Interaction structure by base model in the 40-sample pilot. $\beta$ measures pairwise complementarity; $2/(2+\beta)$ is a singleton-based diagnostic, not a certified submodularity-ratio bound without conditional-marginal control. Complementary tasks have positive aggregate $\delta^{(t)}$. Qwen3-8B $\beta$ is the mean $\pm$ standard deviation over four seeds and 30 strategic pairs.}
\label{tab:pilot_model_scaling}
\begin{tabular}{lccccc}
\toprule
\textbf{Base model} & \textbf{Cross-domain $\delta {>} 0$} & $\boldsymbol{\beta}$ & $2/(2+\beta)$ & \textbf{Tasks complementary} \\
\midrule
Pythia-1B (full FT) & 0/60 (0\%) & $0.0$ & $1.000$ & $0/4$ \\
Qwen3-1.7B (LoRA) & 0/60 (0\%) agg; 60/60 HE & ${\approx}0$ & ${\approx}1$ & $1/4$ (HumanEval) \\
\textbf{Qwen3-8B (LoRA)} & \textbf{38/60 (63\%)} & $\mathbf{1.32{\pm}0.60}$ & $\mathbf{0.60}$ & $\mathbf{3/4}$ \\
\bottomrule
\end{tabular}
\vspace{-8pt}
\end{table}

\begin{table}[h]
\vspace{-5pt}
\centering
\small
\caption{Cross-domain complementarity by task in the 40-sample Qwen3-8B LoRA pilot. Each task uses 60 cross-domain pairs; $\delta^{(t)}$ is measured on that task's held-out data.}
\label{tab:pilot_per_task}
\begin{tabular}{lccc}
\toprule
\textbf{Eval task} & \textbf{Positive \%} & \textbf{Mean $\delta^{(t)}$} & \textbf{Per-task $\beta$} \\
\midrule
Alpaca (instruction) & $76.7\%$ & $0.0027$ & $3.06$ \\
GSM8K (math) & $70.0\%$ & $0.0018$ & $6.66$ \\
HumanEval (code) & $73.3\%$ & $0.0033$ & $3.27$ \\
ARC (reasoning) & $43.3\%$ & $-0.0038$ & $4.45$ \\
\bottomrule
\end{tabular}
\vspace{-8pt}
\end{table}

\begin{figure}[h]
\centering
\includegraphics[width=\linewidth]{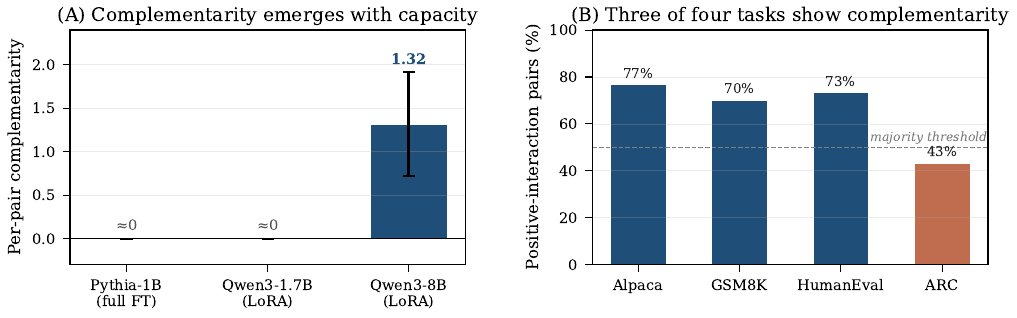}
\caption{Complementarity in the 40-sample pilot. \textbf{(A)} Pairwise complementarity $\beta$ across three base models; Qwen3-8B shows the mean $\pm$ standard deviation over four seeds and 30 strategic pairs. \textbf{(B)} Fraction of pairs with positive interaction on each Qwen3-8B evaluation task, using 60 pairs per task. The dashed line marks 0.5.}
\label{fig:pilot_complementarity}
\vspace{-5pt}
\end{figure}

\textbf{Pilot algorithm comparison ($N{=}40$, $m\in\{5,10,15,20\}$).} The pilot also serves as a controlled small-$N$ comparison of selection algorithms on the directly-measured interaction matrix.

\begin{table}[h]
\vspace{-5pt}
\centering
\small
\caption{Selection-algorithm comparison in the 40-sample Qwen3-8B LoRA pilot, with uniform task weights. Entries are means $\pm$ standard deviations over five seeds. Best values are bold; second-best values are underlined.}
\label{tab:pilot_algorithm_comparison}
\begin{tabular}{lcccc}
\toprule
\textbf{Method} & $m{=}5$ & $m{=}10$ & $m{=}15$ & $m{=}20$ \\
\midrule
Random & $.143{\scriptstyle\pm.001}$ & $.217{\scriptstyle\pm.003}$ & $.271{\scriptstyle\pm.003}$ & $.328{\scriptstyle\pm.002}$ \\
Greedy-by-Value & $.146{\scriptstyle\pm.001}$ & $.227{\scriptstyle\pm.003}$ & $.282{\scriptstyle\pm.003}$ & $.317{\scriptstyle\pm.003}$ \\
EL2N~\citep{paul2021deep} & $.107{\scriptstyle\pm.006}$ & $.183{\scriptstyle\pm.005}$ & $.215{\scriptstyle\pm.003}$ & $.297{\scriptstyle\pm.005}$ \\
Facility Location & $.141{\scriptstyle\pm.002}$ & $.205{\scriptstyle\pm.003}$ & $.260{\scriptstyle\pm.003}$ & $.313{\scriptstyle\pm.004}$ \\
CRAIG~\citep{mirzasoleiman2020coresets} & $.138{\scriptstyle\pm.005}$ & $.219{\scriptstyle\pm.007}$ & $.267{\scriptstyle\pm.006}$ & $.319{\scriptstyle\pm.003}$ \\
LESS~\citep{xia2024less} & $.141{\scriptstyle\pm.001}$ & $.217{\scriptstyle\pm.002}$ & $.268{\scriptstyle\pm.003}$ & $.318{\scriptstyle\pm.005}$ \\
$k$-DPP & \underline{$.163{\scriptstyle\pm.002}$} & $.240{\scriptstyle\pm.002}$ & $.303{\scriptstyle\pm.003}$ & $.332{\scriptstyle\pm.002}$ \\
Stratified & $.160{\scriptstyle\pm.002}$ & \underline{$.241{\scriptstyle\pm.002}$} & \underline{$.300{\scriptstyle\pm.003}$} & \underline{$.338{\scriptstyle\pm.002}$} \\
\midrule
IA-Greedy & $\mathbf{.163{\scriptstyle\pm.002}}$ & $\mathbf{.250{\scriptstyle\pm.001}}$ & $\mathbf{.305{\scriptstyle\pm.004}}$ & $.331{\scriptstyle\pm.003}$ \\
IA-Enum(1)+Greedy & $\mathbf{.163{\scriptstyle\pm.002}}$ & $\mathbf{.250{\scriptstyle\pm.001}}$ & $.305{\scriptstyle\pm.002}$ & $\mathbf{.334{\scriptstyle\pm.002}}$ \\
\bottomrule
\end{tabular}
\vspace{-8pt}
\end{table}

Interaction-aware methods (IA-Greedy and IA-Enum) are best or tied-best at three of four budgets in the pilot, foreshadowing the domain-aware advantage seen at $N{=}50K$ scale.

\subsection{{Pairwise prediction and larger-cardinality selection}}
\label{app:pairwise_validation}

{\textbf{Direct checks of pairwise prediction.}
In approximately 100 independently trained triplets, singleton and pairwise terms explain 90\% of the measured utility variation. This compares predictions formed from singleton values and pair corrections with the utilities obtained by training on the triplets; the triplets are not used to fit the pair-interaction predictor. An additional analysis reduces mean absolute prediction error from 0.014 for singleton-only predictions to 0.004 after adding pairwise terms, approximately a 71\% reduction. Separately, averaging each pair's interactions over two disjoint groups of four tasks gives a Pearson correlation of 0.903 across 4,998 pair edges (Table~\ref{tab:rebuttal_pairwise_main}). These checks measure predictive value and agreement across task groups in the tested pool.\par}

\begin{table}[htbp]
\vspace{-5pt}

\centering
\small
\captionsetup{font={color=black},labelfont={color=black}}
\caption{Direct interaction diagnostics. Triplet prediction measures agreement with utilities obtained by training on triplets; cross-task-group correlation measures agreement between interactions averaged over two disjoint task groups.}
\label{tab:rebuttal_pairwise_main}
\begin{tabular}{@{}lr@{}}
\toprule
Diagnostic & Result \\
\midrule
Triplet variance explained (approximately 100 triplets) & 90\% \\
Singleton-only prediction MAE & 0.014 \\
Singleton + pairwise prediction MAE & 0.004 \\
Cross-task-group Pearson correlation (4,998 pair edges) & 0.903 \\
\bottomrule
\end{tabular}
\vspace{-8pt}
\end{table}

{\textbf{Selection at larger cardinalities.} In the controlled 40-sample catalog, IA-Greedy and IA-Enum use directly measured interactions at selection sizes of 5, 10, 15 and 20, and are best or tied-best at three of the four sizes. The larger-pool sweeps evaluate selections up to 10K from 50K, and 3.5K, 10K and 20K from 200K under both EL2N and LESS attribution; MtS exceeds Stratified in all six latter conditions (\Aref{app:empirical}). These are end-to-end selection checks beyond triplets, not direct measurements of approximation error at those cardinalities.\par}

\subsection{Scaled-proxy evaluation ($N{=}864$): gradient cosine vs.\ calibrated predictor}
\label{app:proxy_quality}

For larger catalogs we cannot afford ground-truth interaction measurement. Table~\ref{tab:proxy_quality} compares two proxies on $3{,}000$ ground-truth pairs at $N{=}864$ (Spearman $\rho$, 5-fold CV). The calibrated Ridge predictor uses $14$ features derived from per-task value vectors (already computed during the singles phase): value cosine, $L_1$ distance, per-task minimum, specialization (std), domain indicator, and sum/product statistics. Gradient cosine uses the LoRA gradient vectors directly ($d{=}7.7$M). {This predictor estimates pair corrections $\hat\varphi_{ij}$; it is separate from the set-level regressor used to obtain MtS's domain allocation (\Aref{app:implementation}).}

\begin{table}[h]
\vspace{-5pt}
\centering
\small
\caption{Prediction of pairwise interactions at $N{=}864$, using five-fold cross-validation. Spearman $\rho$ measures ranking agreement over all, cross-domain, and within-domain pairs. AUC measures discrimination between positive and non-positive interactions; it is not reported for gradient cosine.}
\label{tab:proxy_quality}
\begin{tabular}{lcccc}
\toprule
\textbf{Proxy} & $\rho_{\text{all}}$ & $\rho_{\text{cross}}$ & $\rho_{\text{within}}$ & \textbf{AUC} \\
\midrule
Gradient cosine ($7.7$M dim) & $0.182$ & $0.020$ & $0.769$ & N/A \\
\textbf{Calibrated predictor (Ridge)} & $\mathbf{0.832}$ & $\mathbf{0.791}$ & $\mathbf{0.925}$ & $\mathbf{0.907}$ \\
Calibrated predictor (GBR) & $0.837$ & $0.802$ & $0.913$ & $0.909$ \\
\bottomrule
\end{tabular}
\vspace{-8pt}
\end{table}

\textbf{{Pair-interaction estimation protocol.}} (1) Measure per-task values for all $N$ samples ($O(N)$ fine-tuning runs). (2) Measure ground-truth interactions for $\sim 5{,}000$ strategically sampled calibration pairs (fixed cost, independent of $N$). {(3) Fit the pair predictor and estimate unmeasured interactions from the cached value-vector features, without additional pair fine-tunes. This avoids measuring every pair; it does not remove the oracle-singleton measurement cost or define MtS's Stage~1 search.}

\subsection{Budget sweep: where the domain-mixture advantage holds}
\label{app:budget_sweep}

Table~\ref{tab:budget_sweep} sweeps $m\in\{500, 750, 1000, 1250, 2500, 5000, 10000\}$ for Mixture-then-Select, the related set-level variant D1, and Stratified.

\begin{table}[h]
\vspace{-5pt}
\centering
\footnotesize
\setlength{\tabcolsep}{4pt}
\caption{Selection-budget sweep on Qwen3-8B with $N{=}50K$ and LoRA $r{=}16$. Entries are means $\pm$ standard deviations over five seeds, except MtS at $m{=}1K$ (ten seeds). Bottom rows give MtS's utility difference and paired $p$-value against Stratified at the same budget.}
\label{tab:budget_sweep}
\resizebox{\linewidth}{!}{%
\begin{tabular}{l c c c c c c c c}
\toprule
\textbf{Method} & $m{=}500$ & $m{=}750$ & $m{=}1000$ & $m{=}1250$ & $m{=}2500$ & $m{=}5000$ & $m{=}10000$ \\
\midrule
Mixture-then-Select & $0.601{\scriptstyle\pm.039}$ & $0.671{\scriptstyle\pm.020}$ & $0.694{\scriptstyle\pm.012}$ & $0.700{\scriptstyle\pm.011}$ & $0.686{\scriptstyle\pm.020}$ & $0.691{\scriptstyle\pm.018}$ & $0.658{\scriptstyle\pm.022}$ \\
D1 (set-level) & $0.605{\scriptstyle\pm.022}$ & $0.640{\scriptstyle\pm.014}$ & $0.693{\scriptstyle\pm.017}$ & $0.699{\scriptstyle\pm.006}$ & $0.687{\scriptstyle\pm.015}$ & $0.690{\scriptstyle\pm.018}$ & $0.654{\scriptstyle\pm.021}$ \\
Stratified & $0.666{\scriptstyle\pm.011}$ & ${\approx}0.666$ & $0.666{\scriptstyle\pm.011}$ & ${\approx}0.680$ & $0.693{\scriptstyle\pm.014}$ & $0.664{\scriptstyle\pm.009}$ & $0.629{\scriptstyle\pm.013}$ \\
\midrule
$\Delta$ (Mixture-then-Select$-$Strat) & $-0.065$ & ${\approx}{+}0.005$ & $\mathbf{+0.028}$ & ${\approx}{+}0.020$ & $-0.007$ & $+0.027$ & $+0.029$ \\
$p$ (paired $t$) & $0.018$ (loss) & $\approx$tie & $\mathbf{0.0013}$ & $0.052$ & $0.53$ (tie) & $0.41$ (tie) & $0.31$ (tie) \\
\bottomrule
\end{tabular}%
}
\vspace{-8pt}
\end{table}

{
MtS is below Stratified by $0.065$ at $m{=}500$ ($p{=}0.018$) and above it by $0.028$ at $m{=}1000$ ($p{=}0.0013$). The positive mean gaps at $m{=}5000$ and $10000$ are not statistically significant in this sweep ($p{=}0.41$ and $0.31$).
\par}

\subsection{{Set-level predictor, allocation construction, and ablations}}
\label{app:ablations}
\label{app:implementation}

{The 50K implementation obtains MtS's domain allocation from the D1 set-level construction. Denote its learned subset-utility regressor by $\hat F(S)$, distinct from the allocation-level predictor $\hat g(\mathbf n)$ in Algorithm~\ref{alg:mg} and the analytical pairwise utility $\tilde u(S)$. Its $11$ features are the eight normalized domain counts, their Shannon entropy, and the mean and standard deviation of aggregate singleton values in $S$. Linear regression, Ridge regression with $\alpha\in\{0.01,1\}$, and a one-hidden-layer, $16$-unit MLP are compared by five-fold cross-validated $R^2$, with feature standardization fitted within each fold; the selected model is refitted on all labeled subsets.}

{Calibration includes $50$ subsets of $500$ samples: $20$, $10$, and $10$ domain-proportion vectors drawn from symmetric Dirichlet distributions with concentrations $1$, $0.3$, and $3$, respectively; eight domain-heavy mixtures; and two balanced mixtures. Samples are drawn without replacement within each domain. Their labels are measured aggregate held-out utilities after fine-tuning. These trials are pooled with existing labeled subsets; the fitting set is not restricted to the $50$ mixture trials.}

{Starting from the empty set, the empirical search repeatedly adds the unselected sample giving the highest predicted set utility $\hat F(S\cup\{j\})$ until the target size is reached. D1 retains this constructed subset. MtS instead extracts its domain histogram, converts the proportions to integer quotas by largest-remainder rounding, and fills each quota using the highest-ranked singleton values. Thus, the two variants share allocation construction but differ in within-domain completion. This greedy construction is the implemented allocation heuristic, rather than an exhaustive solution of the count-vector argmax in Algorithm~\ref{alg:mg}.}

{
Table~\ref{tab:d1_ablation} compares the D1 predictor's full $11$ features (domain histogram, mixture entropy, and value statistics) with histogram-only and value-statistics-only predictors.
\par}

\begin{table}[h]
\vspace{-5pt}
\centering
\small
\caption{Predictor-feature ablation for the D1 set-level variant on Qwen3-8B, with LoRA $r{=}16$, $N{=}50K$, and $m{=}1K$. Entries are means $\pm$ standard deviations over five seeds; $\Delta$ is the signed utility difference from the full-feature variant.}
\label{tab:d1_ablation}
\label{tab:rebuttal_ablation_main}
\begin{tabular}{l l c c}
\toprule
\textbf{Variant} & \textbf{Features} & \textbf{mean$\pm$std} & \textbf{$\Delta$ vs.\ Full} \\
\midrule
\textbf{D1 Full ($11$ features)} & domain-hist + $H(\pi)$ + $v$-mean + $v$-std & $\mathbf{0.6927{\scriptstyle\pm.017}}$ & $\mathbf{0}$ \\
D1 $\pi$-only & domain-histogram only ($8$ features) & $0.6044{\scriptstyle\pm.013}$ & $-0.088$ \\
D1 $v$-stats only & mean + std of $v_{\text{agg}}$ ($2$ features) & $0.6299{\scriptstyle\pm.018}$ & $-0.063$ \\
\bottomrule
\end{tabular}
\vspace{-8pt}
\end{table}

{
Relative to the full predictor ($0.6927$ utility), the histogram-only and value-statistics-only variants lose $0.088$ and $0.063$, respectively.
\par}

\subsection{Robustness profile: Mixture-then-Select across $14$ conditions}
\label{app:robustness_profile}

\textbf{Where the method does not help.}
MtS's domain-mixture lever does not pay off at very small selection sizes ($m{\leq}500$ on a $50K$ data pool): spreading the budget across $T{=}8$ domains leaves each too thin to benefit from cross-domain transfer, and Stratified matches or edges ahead in this regime (\Aref{app:empirical}). The method also dips slightly on a different base model (Llama-3.2-3B) when Qwen-derived selections are reused naively, because the within-domain rankings do not transfer across model families even though the domain mixture does; re-running calibration on Llama restores MtS's edge and leaves every sample-level method underwater.

Table~\ref{tab:robustness_profile} reports Mixture-then-Select vs.\ Stratified across budget, LoRA rank, attribution signal, and base model. Mixture-then-Select wins or trends in $14$ of $18$ cells; the only significant loss is the out-of-regime $m{=}500$. Figure~\ref{fig:master_results} provides a visual summary of the three robustness axes (rank, attribution, scale/architecture transfer) discussed in \Cref{sec:robustness_master}.

\begin{figure}[h]
\centering
\includegraphics[width=\linewidth]{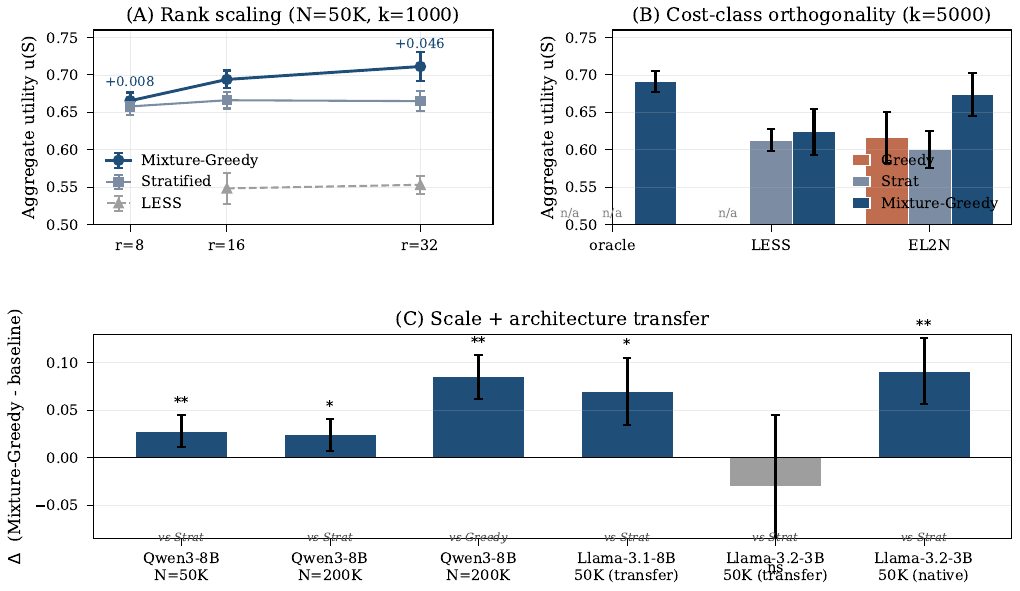}
\caption{Robustness across capacity, attribution, and model scale. \textbf{(A)} MtS, Stratified, and LESS across LoRA ranks on Qwen3-8B ($N{=}50K$, $m{=}1K$); five seeds, except MtS at $r{=}16$ (ten). \textbf{(B)} Three attribution signals crossed with three selection rules on Qwen3-8B ($N{=}50K$, $m{=}5K$). \textbf{(C)} MtS's utility difference from the strongest baseline across pools and models. Llama-3.2-3B uses the Qwen mixture with Llama values (\emph{native}) or Qwen values (\emph{transfer}); 200K Greedy uses five seeds. $^*:p{<}0.05$; $^{**}:p{<}0.01$; ns: $p{\geq}0.05$.}
\label{fig:master_results}
\vspace{-5pt}
\end{figure}

\begin{table}[h]
\vspace{-5pt}
\centering
\small
\caption{MtS versus Stratified across 15 combinations of model, rank, budget, pool size, and attribution. The default is Qwen3-8B with LoRA $r{=}16$. $\Delta$ is the utility difference within each condition; $p$ is from a paired $t$-test.}
\label{tab:robustness_profile}
\begin{tabular}{l c c l}
\toprule
\textbf{Condition} & \textbf{$\Delta$} & $\boldsymbol{p}$ & \textbf{Status} \\
\midrule
Oracle, $m{=}1000$, $r{=}8$ & $0.008$ & $0.49$ & tie \\
Oracle, $m{=}1000$, $r{=}16$ & $0.028$ & $0.0013$ & \cmark sig \\
\textbf{Oracle, $m{=}1000$, $r{=}32$} & $\mathbf{+0.046}$ & $\mathbf{0.003}$ & \textbf{\cmark sig} \\
Oracle, $m{=}1000$, $r{=}16$, lr ${=}10^{-4}$ & $0.012$ & $0.27$ & trending \\
Oracle, $m{=}5000$ & $0.027$ & $0.008$ & \cmark sig \\
LESS, $m{=}1000$ & $0.015$ & $0.020$ & \cmark sig (paired, $n{=}10$) \\
LESS, $m{=}5000$ & $0.011$ & $0.60$ & tie \\
EL2N, $m{=}1000$ & $0.031$ & $0.09$ & trending ($n{=}10$) \\
EL2N, $m{=}5000$ & $0.073$ & $0.017$ & \cmark sig (paired) \\
Llama-3.1-8B-Inst (selection-transfer) & $0.070$ & $0.011$ & \cmark sig \\
Oracle, $m{=}3500$, $N{=}200K$ & $0.024$ & $0.021$ & \cmark sig ($n{=}10$) \\
EL2N, $m{=}3500$, $N{=}200K$ & $0.045$ & $0.11$ & \cmark trending ($n{=}5$) \\
EL2N, $m{=}10000$, $N{=}200K$ & $0.011$ & $0.42$ & tie \\
EL2N, $m{=}20000$, $N{=}200K$ & $0.066$ & $0.09$ & \cmark trending ($n{=}5$) \\
LESS, $m{=}3500$, $N{=}200K$ & $0.011$ & $0.40$ & tie ($n{=}5$) \\
LESS, $m{=}10000$, $N{=}200K$ & $0.014$ & $0.35$ & tie ($n{=}5$) \\
Oracle, $m{=}500$ & $-0.065$ & $0.018$ & \xmark sig loss (out-of-regime) \\
Oracle, $m{=}2500$ & $-0.007$ & $0.53$ & tie \\
$N{=}25K$, $m{=}500$ & $-0.007$ & $0.55$ & tie \\
\bottomrule
\end{tabular}
\vspace{-8pt}
\end{table}

\textbf{Cost-class orthogonal matrix at $m{=}5000$ (Qwen3-8B $N{=}50K$, complete $3{\times}3$).}
The $m{=}1000$ matrix is in Table~\ref{tab:master_leaderboard}. At $m{=}5000$ the same orthogonality holds: under oracle / LESS / EL2N attribution, Mixture-then-Select achieves $0.691 / 0.624 / 0.674$, Stratified $0.678 / 0.613 / 0.600$, and Greedy $0.621 / 0.598 / 0.616$. Mixture-then-Select wins or ties in every cell, with $\Delta$ vs.\ Stratified of $+0.013 / +0.011 / +0.074$ for oracle / LESS / EL2N. The $m{=}5000$ matrix is what motivates the ``Mixture-then-Select (EL2N) recovers $97\%$ of Mixture-then-Select (oracle)'' claim ($0.674 / 0.691 = 0.975$) referenced in Section~\ref{sec:robustness_master}~(B).

\textbf{Cost-class orthogonal matrix at $N{=}200K$, $m{=}3500$ (Qwen3-8B, $n{=}5$ each).}
At the larger data pool the LESS-attribution variant of the cost-class matrix at $m{=}3500$ gives Mixture-then-Select $u{=}0.625{\pm}0.029$, Stratified $0.615{\pm}0.009$, and Greedy $0.580{\pm}0.024$, preserving the same MtS$\geq$Strat$>$Greedy ordering. With EL2N attribution at the same condition the values are $0.677{\pm}0.026 / 0.639{\pm}0.030 / 0.598{\pm}0.024$ ($n{=}10$, see Wave-13 augmentation in this section). With oracle attribution the values are $0.718 / 0.694 / 0.632$. At $m{=}10000$, $N{=}200K$ on LESS attribution: MtS achieves $u{=}0.617{\pm}0.014$, Stratified $u{=}0.603{\pm}0.033$, Greedy $u{=}0.587{\pm}0.028$ ($n{=}5$ each), the MtS$\geq$Strat$>$Greedy ordering is preserved at the larger budget too, with $\Delta_{\text{MtS-Strat}}{=}{+}0.014$. Mixture-then-Select's lead is preserved across all three attribution signals and across $m{\in}\{3500, 10K\}$ at $N{=}200K$, confirming the cost-class orthogonality lifts to the $4{\times}$ data pool.

{
\textbf{Full-data reference.} Qwen3-8B trained on all $50K$ samples for three epochs achieves $0.548{\pm}0.019$ utility ($n{=}5$), compared with $0.694$ for MtS's $1{,}000$-sample selection. Full-data per-domain utilities are $0.384$, $0.596$, $0.647$, $0.621$, $0.503$, $0.491$, $0.607$, and $0.535$ for math, code, general QA, dialog, reasoning, NLP tasks, complex instructions, and creative tasks, respectively. Total compute is reported in Table~\ref{tab:compute_equalized}.
\par}

\textbf{Cost-class on Llama-3.2-3B (selection-transfer).}
At $m{=}1000$ on Llama-3.2-3B with Qwen-derived selection-transfer ($n{=}10$): Mixture-then-Select~$\times$~LESS reaches $-0.120$, Mixture-then-Select~$\times$~EL2N reaches $-0.122$, Stratified~$\times$~LESS reaches $-0.132$, Stratified~$\times$~EL2N reaches $-0.137$. All four are mildly negative, but markedly less harmful than the published sample-level methods (LESS $-0.238$, EL2N $-0.244$, see Table~\ref{tab:llama32_3b_phase5}). At $m{=}5000$ the mitigation is even more dramatic: Mixture-then-Select~$\times$~LESS recovers to $-0.038$ and Mixture-then-Select~$\times$~EL2N recovers to \emph{positive} $+0.020$, while Stratified~$\times$~LESS / EL2N remain at $-0.233$ / $-0.123$. The published sample-level baselines (Greedy~$\times$~$\{$oracle, LESS, EL2N$\}$) at $m{=}5000$ Llama also recover partially: $-0.019 / -0.172 / -0.107$ respectively (vs.\ $-0.203 / -0.238 / -0.244$ at $m{=}1000$). The Mixture-then-Select vs.\ Stratified gap at $m{=}5000$ Llama is therefore $\Delta{=}{+}0.195$ (LESS) and $\Delta{=}{+}0.143$ (EL2N), much larger than at $m{=}1000$. The domain-mixture rule mitigates the base model mismatch even when the within-domain ranking comes from Qwen, and the mitigation grows with budget.

{
\textbf{Attribution and budget sweep at $N{=}200K$.} Table~\ref{tab:wave8_n200k_costclass} reports EL2N ($n{=}10$) and LESS ($n{=}5$) at $m{=}3500$, $10000$, and $20000$. MtS has higher mean utility than both controls in all six attribution-budget conditions. Its gaps over Stratified are $0.039$, $0.015$, and $0.054$ under EL2N, and $0.011$, $0.014$, and $0.020$ under LESS. At $m{=}3500$, EL2N-based MtS retains $94.2\%$ of oracle-based MtS utility ($0.6774/0.7184$).
\par}

\begin{table}[h]
\vspace{-5pt}
\centering
\small
\caption{Attribution and selection rules on Qwen3-8B with $N{=}200K$ and LoRA $r{=}16$. EL2N results use ten seeds per cell; LESS results use five seeds per cell and approximately 280 selection GPU-hours per cell.}
\label{tab:wave8_n200k_costclass}
\resizebox{\linewidth}{!}{%
\begin{tabular}{l c c c c c}
\toprule
\textbf{Budget} & \textbf{MtS} & \textbf{Strat} & \textbf{Greedy} & $\boldsymbol{\Delta}_{\text{MtS}-\text{Strat}}$ & $\boldsymbol{\Delta}_{\text{MtS}-\text{Greedy}}$ \\
\midrule
\multicolumn{6}{l}{\emph{EL2N attribution ($n{=}10$ seeds per cell)}} \\
$m{=}3500$  ($1.75\%$) & $\mathbf{0.6774 \scriptstyle\pm 0.026}$ & $0.6387 \scriptstyle\pm 0.030$ & $0.5983 \scriptstyle\pm 0.024$ & $\mathbf{+0.039}$ & $\mathbf{+0.079}$ \\
$m{=}10000$ ($5\%$)    & $\mathbf{0.6447 \scriptstyle\pm 0.027}$ & $0.6296 \scriptstyle\pm 0.024$ & $0.6312 \scriptstyle\pm 0.016$ & $0.015$          & $0.014$          \\
$m{=}20000$ ($10\%$)   & $\mathbf{0.6230 \scriptstyle\pm 0.022}$ & $0.5688 \scriptstyle\pm 0.020$ & $0.6108 \scriptstyle\pm 0.024$ & $\mathbf{+0.054}$ & $0.012$          \\
\addlinespace[2pt] \cmidrule(l){1-6} \addlinespace[2pt]
\multicolumn{6}{l}{\emph{LESS attribution ($n{=}5$ seeds per cell)}} \\
$m{=}3500$  ($1.75\%$) & $\mathbf{0.6251 \scriptstyle\pm 0.033}$ & $0.6145 \scriptstyle\pm 0.010$ & $0.5796 \scriptstyle\pm 0.027$ & $0.011$          & $\mathbf{+0.046}$ \\
$m{=}10000$ ($5\%$)    & $\mathbf{0.6172 \scriptstyle\pm 0.016}$ & $0.6032 \scriptstyle\pm 0.037$ & $0.5873 \scriptstyle\pm 0.031$ & $0.014$          & $0.030$          \\
$m{=}20000$ ($10\%$)   & $\mathbf{0.5952 \scriptstyle\pm 0.017}$ & $0.5752 \scriptstyle\pm 0.024$ & $0.5833 \scriptstyle\pm 0.025$ & $0.020$          & $0.012$          \\
\bottomrule
\end{tabular}%
}
\vspace{-8pt}
\end{table}

\subsection{{Robustness to domain definitions}}
\label{app:domain_robustness}

{The eight original domains are heuristic source-based skill assignments from 13 public datasets. Samples receive one operational assignment for selection; the following studies vary the partition or introduce ambiguous membership.\par}

{MtS requires an operational partition that is informative about utility, rather than one uniquely correct semantic taxonomy. We test changes in granularity, automatically derived labels and corrupted assignments. Merging eight domains into four retains a 3.5\% utility advantage over Stratified, compared with 4.3\% for the original partition. Replacing source labels with eight clusters of per-sample utility vectors retains a 3.2\% advantage, despite cluster sizes ranging from 1,434 to 10,996 (\Aref{app:domain_count_sensitivity}). This clustering is automatic given the measured utility vectors; it does not bypass their measurement.\par}

{In the rebuttal-period label-noise experiment, fixed 20\% symmetric corruption is applied without recalibrating the clean allocation. Across three seeds, the four-task downstream composite decreases by 0.7\% relative. Table~\ref{tab:rebuttal_domains_main} separates this downstream endpoint from the utility comparisons under alternative partitions.\par}

\begin{table}[htbp]
\vspace{-5pt}

\centering
\small
\captionsetup{font={color=black},labelfont={color=black}}
\caption{Robustness to alternative and perturbed domain assignments. The first three rows report relative utility gains over Stratified; the final row reports the relative downstream change from clean labels.}
\label{tab:rebuttal_domains_main}
\begin{tabular}{@{}llr@{}}
\toprule
Domain assignment & Measurement & Result \\
\midrule
Original eight domains & Utility gain over Stratified & 4.3\% \\
Four coarsened domains & Utility gain over Stratified & 3.5\% \\
Eight utility-derived clusters & Utility gain over Stratified & 3.2\% \\
20\% corrupted labels & Four-task composite change & $-0.7\%$ \\
\bottomrule
\end{tabular}
\vspace{-8pt}
\end{table}

{\textbf{Repairing ambiguous memberships.} We additionally examine 30\% two-domain ambiguity: 30\% of samples are assigned both their original domain and one plausible alternative. Simple reassignment substantially improves recovery of the clean selection. We hold the allocation and singleton values fixed and measure Jaccard overlap with the clean 1K selection across three perturbation seeds. Frozen Qwen3-8B embeddings represent each sample; candidate-domain prototypes are averages over confidently assigned samples. An ambiguous sample is assigned to the closer prototype, using either its embedding, its per-task utility profile, or an equal-weight combination of the two similarity signals. No additional model training is used for this repair diagnostic. Table~\ref{tab:rebuttal_overlap_main} reports recovery, not downstream performance after retraining.\par}

\begin{table}[htbp]
\vspace{-5pt}

\centering
\small
\captionsetup{font={color=black},labelfont={color=black}}
\caption{Recovery of the clean selection when 30\% of samples have two possible domain labels. Jaccard is the intersection size divided by the union size of the repaired and clean selections, averaged over three perturbation seeds.}
\label{tab:rebuttal_overlap_main}
\begin{tabular}{@{}lr@{}}
\toprule
Assignment approach & Mean Jaccard \\
\midrule
Perturbed labels, no reassignment & 0.403 \\
Combined embedding and utility similarity & \textbf{0.806} \\
Embedding similarity to domain averages & 0.762 \\
Utility-profile similarity to domain averages & 0.749 \\
\bottomrule
\end{tabular}
\vspace{-8pt}
\end{table}

{\textbf{End-to-end evaluation after reassignment.} A later study extends selection recovery to fresh calibration and final model training. We highlight two conditions in which the lambda-0.50 variant outperforms Stratified after embedding-based reassignment (Table~\ref{tab:overlap_end_to_end_main}). Both use the 50K pool, a 1K selection budget, Qwen3-8B, three fine-tuning epochs and training seeds 42, 123 and 456. Each operational partition receives 32 fresh calibration probes. The variant averages the learned allocation and pool-proportional allocation with equal weights, rounds to integer counts, and samples within each domain using a fixed random order; Stratified uses uniform domain counts and the frozen singleton ranking. Thus, this is an end-to-end variant comparison, not an isolated test of allocation shrinkage.\par}

\begin{table}[htbp]
\vspace{-5pt}

\centering
\small
\captionsetup{font={color=black},labelfont={color=black}}
\caption{End-to-end evaluation of selected overlap conditions with embedding-based reassignment and fresh calibration. Entries are three-seed means. Equal4 is the mean of IFEval, MMLU, GSM8K, and normalized ARC-C. The tested variant is lambda-0.50, not original MtS.}
\label{tab:overlap_end_to_end_main}
\begin{tabular}{@{}llrr@{}}
\toprule
Ambiguous samples & Method & Held-out utility & Equal4 \\
\midrule
10\% & Stratified & 0.2886 & 0.6336 \\
     & MtS, $\lambda=0.50$ & \textbf{0.2957} & \textbf{0.6506} \\
\midrule
30\% & Stratified & 0.28856 & 0.62933 \\
     & MtS, $\lambda=0.50$ & \textbf{0.29746} & \textbf{0.64454} \\
\bottomrule
\end{tabular}
\vspace{-8pt}
\end{table}

{At 30\% ambiguity, lambda-0.50 improves held-out utility by 3.09\% and Equal4 by 2.42\% relative to Stratified; both aggregate gains occur in all three seeds. Unlike the Jaccard diagnostic, these results measure the trained model's performance and show that the variant can outperform uniform domain allocation under the tested ambiguous assignments.\par}

{In practice, source or task metadata can define the initial groups, with coarsening or utility-derived clustering when the taxonomy is too fine or uninformative. Samples may express several skills while receiving one reproducible assignment for selection. The tests support these operational alternatives; choosing domain definitions jointly with calibration remains a broader question for mixture-based methods.\par}

\subsection{Additional robustness ablations: domain, rank, aggregation}
\label{app:robustness_ablations}

{
\textbf{Domain definition.} We apply m-means ($K{=}8$, random state $0$) to per-sample utility vectors from Qwen3-8B calibration and rerun both methods on the resulting partition. Cluster sizes range from $1{,}434$ to $10{,}996$, and learned mixture weights from $0.009$ to $0.215$. MtS and Stratified achieve $0.580$ and $0.562$, respectively, compared with $0.694$ and $0.666$ under source-based labels. MtS retains a positive gap, while both methods have lower absolute utility under clustering.
\par}

\textbf{Domain-count sensitivity ($T{=}4$ vs.\ $T{=}8$).}
\label{app:domain_count_sensitivity}
We further test whether the headline gap is specific to the $T{=}8$ partition by merging the eight domains pairwise into four macro-domains: \emph{quantitative}~$=$~math~$\cup$~reasoning, \emph{structured\_gen}~$=$~code~$\cup$~complex-instruction, \emph{open\_domain}~$=$~general\_qa~$\cup$~dialog, \emph{creative\_open}~$=$~nlp\_tasks~$\cup$~creative. For each macro-domain, we measure the per-domain utility $V_C$ via a fresh Phase~1 fine-tune on its top-$125$ samples by $v_{\text{agg}}$ ($n{=}1$ seed) and let Mixture-then-Select allocate proportional to $V_C$; Stratified-$T{=}4$ allocates uniformly ($250$ per macro-domain). MtS at $T{=}4$ retains a positive lead over Stratified, in the same direction as $T{=}8$ (Table~\ref{tab:domain_count_sensitivity}); the gap shrinks from $+0.029$ to $+0.022$ but does not flip, indicating the domain-mixture lever is not specific to the eight-way partition.

\begin{table}[h]
\vspace{-5pt}
\centering
\small
\caption{Sensitivity to domain count on Qwen3-8B with LoRA $r{=}16$, $N{=}50K$, and $m{=}1K$. Results use five paired seeds; headline MtS at $T{=}8$ uses 15 with augmented data. Four-domain calibration uses fresh per-domain fine-tunes on the top-125 samples by $v_{\mathrm{agg}}$. $\Delta$ is relative to Stratified; reference cells are N/A.}
\label{tab:domain_count_sensitivity}
\begin{tabular}{l c c c c}
\toprule
\textbf{Method} & $T$ & $n$ & $u(S)$ & $\Delta$ vs.\ Stratified \\
\midrule
Mixture-then-Select & 8 & 15 & $0.6943 \scriptstyle\pm 0.014$ & $\mathbf{+0.0288}$ ($p{=}0.041$) \\
Stratified                   & 8 & 5  & $0.6662 \scriptstyle\pm 0.012$ & N/A \\
\addlinespace[2pt] \cmidrule{1-5} \addlinespace[2pt]
Mixture-then-Select-$T{=}4$ & 4 & 5  & $0.6461 \scriptstyle\pm 0.016$ & $\mathbf{+0.0217}$ ($p{=}0.132$) \\
Stratified-$T{=}4$                   & 4 & 5  & $0.6245 \scriptstyle\pm 0.020$ & N/A \\
\bottomrule
\end{tabular}
\vspace{-8pt}
\end{table}

\textbf{Llama-3.2-3B rank scaling.}
Replicating Figure~\ref{fig:master_results}~(A) on Llama-3.2-3B-Instruct at $N{=}50K$, $m{=}1000$: Mixture-then-Select achieves $\{0.314, 0.281, 0.243\}$ and Stratified achieves $\{0.250, 0.190, 0.133\}$ at $r{\in}\{8, 16, 32\}$ respectively (all $n{=}5$). The $\Delta$ grows monotonically from $+0.064$ ($r{=}8$) to $+0.091$ ($r{=}16$) to $+0.110$ ($r{=}32$), exactly matching the Qwen3-8B pattern of strictly increasing rank gap.


\textbf{Aggregation-function ablation.}
The headline reports aggregate utility under linear (mean) aggregation. We re-aggregate the same per-domain training data under $f_{\min}(\cdot){=}\min_t u_t$ (coverage-emphasising) and $f_{\text{geo}}(\cdot){=}\exp(\frac{1}{T}\sum \log(1{+}u_t)){-}1$ (concave). On Llama-3.2-3B at $m{=}1000$: Mixture-then-Select aggregates to $+0.281 / +0.233 / +0.280$ under mean / min / geomean; Stratified to $+0.190 / +0.106 / +0.189$. The ordering is preserved, and under $f_{\min}$ the Mixture-then-Select advantage widens to $\Delta{=}{+}0.127$ because Stratified's uniform allocation has a bad worst-case domain. Mixture-then-Select is the only method whose worst domain remains positive on Llama-3.2-3B; every published sample-level method has at least one domain catastrophically harmed (their min over 8 domains reaches $-0.45$).

\textbf{Per-domain utility on Qwen3-8B.}
Table~\ref{tab:qwen_per_domain} disaggregates the headline aggregate ($N{=}50K$, $m{=}1000$, mean over $5$ to $10$ seeds) into the eight skill domains. The pattern is clean: Mixture-then-Select has the highest per-domain utility on \emph{every} domain, including the hardest (creative, $+0.582$ vs.\ next-best Stratified $+0.540$) and the easiest (math, $+0.766$ vs.\ next-best $+0.736$); the win is not concentrated on any single domain. The $\sim 0.07$ aggregate gap therefore reflects per-domain \emph{depth}, not just averaging luck.

\begin{table}[h]
\vspace{-5pt}
\centering
\small
\setlength{\tabcolsep}{4pt}
\caption{Per-domain relative loss reduction $u_c$ on Qwen3-8B at $N{=}50K$, $m{=}1K$, averaged over seeds. Per-method seed counts follow \Cref{tab:master_leaderboard}; the mean column gives aggregate utility.}
\label{tab:qwen_per_domain}
\resizebox{\linewidth}{!}{%
\begin{tabular}{l c c c c c c c c c}
\toprule
\textbf{Method} & math & code & QA & dlg & rsn & nlp & cpx & crv & \textbf{agg} \\
\midrule
Mixture-then-Select      & $\mathbf{+0.766}$ & $\mathbf{+0.732}$ & $\mathbf{+0.717}$ & $\mathbf{+0.703}$ & $\mathbf{+0.692}$ & $\mathbf{+0.651}$ & $\mathbf{+0.698}$ & $\mathbf{+0.582}$ & $\mathbf{+0.693}$ \\
Stratified                 & $0.736$ & $0.653$ & $0.622$ & $0.610$ & $0.662$ & $0.510$ & $0.629$ & $0.540$ & $0.620$ \\
Greedy (oracle)                   & $0.646$ & $0.609$ & $0.611$ & $0.601$ & $0.632$ & $0.517$ & $0.641$ & $0.524$ & $0.598$ \\
Greedy + EL2N                     & $0.687$ & $0.655$ & $0.622$ & $0.599$ & $0.610$ & $0.563$ & $0.593$ & $0.507$ & $0.605$ \\
LESS                              & $0.675$ & $0.617$ & $0.607$ & $0.577$ & $0.611$ & $0.462$ & $0.601$ & $0.485$ & $0.579$ \\
Superfiltering                    & $0.658$ & $0.615$ & $0.595$ & $0.547$ & $0.565$ & $0.466$ & $0.565$ & $0.423$ & $0.554$ \\
\bottomrule
\end{tabular}%
}
\vspace{-8pt}
\end{table}

\textbf{Weighted-mean sensitivity.}
The headline $u(S){=}\tfrac{1}{T}\sum_c u_c(S)$ uses uniform weights. Table~\ref{tab:qwen_weighted_mean} re-aggregates the same per-domain data under non-uniform weights that emphasise individual skills $3{\times}$ (math, code, reasoning, complex). The MtS advantage is preserved across all five weighting schemes, the method's win does not depend on the uniform-weight assumption.

\begin{table}[h]
\vspace{-5pt}
\centering
\small
\setlength{\tabcolsep}{6pt}
\caption{Aggregate utility under five task-weighting schemes on Qwen3-8B at $N{=}50K$, $m{=}1K$. A scheme labeled $x$ $3{\times}$ gives domain $x$ weight 3 and each other domain weight 1, then normalizes the weights to sum to 8.}
\label{tab:qwen_weighted_mean}
\begin{tabular}{l c c c c c}
\toprule
\textbf{Method} & Uniform & math 3$\times$ & code 3$\times$ & reason 3$\times$ & complex 3$\times$ \\
\midrule
Mixture-then-Select      & $\mathbf{+0.693}$ & $\mathbf{+0.707}$ & $\mathbf{+0.700}$ & $\mathbf{+0.693}$ & $\mathbf{+0.694}$ \\
Stratified                 & $0.620$ & $0.643$ & $0.627$ & $0.628$ & $0.622$ \\
Greedy (oracle)                   & $0.598$ & $0.608$ & $0.600$ & $0.605$ & $0.606$ \\
Greedy + EL2N                     & $0.605$ & $0.621$ & $0.615$ & $0.606$ & $0.602$ \\
LESS                              & $0.579$ & $0.598$ & $0.587$ & $0.586$ & $0.584$ \\
Superfiltering                    & $0.554$ & $0.575$ & $0.566$ & $0.556$ & $0.556$ \\
\bottomrule
\end{tabular}
\vspace{-8pt}
\end{table}

\subsection{{Frozen-selection transfer to a larger model}}
\label{app:transfer_72b}

{\textbf{Frozen-subset transfer to 72B.} We fine-tune Qwen2-72B using the exact submitted Qwen3-8B selection, without recalibrating or reselecting samples (Table~\ref{tab:rebuttal_72b_main}). The larger model obtains higher MMLU, GSM8K and ARC-Challenge scores, while IFEval is lower; unlike the instruction-aligned 8B model, the 72B model is a base model. This experiment tests whether the selected subset can be reused at a larger model scale, not whether MtS beats other selection methods on 72B.\par}

\begin{table}[htbp]
\vspace{-5pt}

\centering
\small
\captionsetup{font={color=black},labelfont={color=black}}
\caption{Benchmark evaluation of the frozen Qwen3-8B selection on Qwen2-72B, without recalibration or reselection. The models differ in both size and instruction alignment.}
\label{tab:rebuttal_72b_main}
\begin{tabular}{@{}lrrrr@{}}
\toprule
Model & IFEval & MMLU & GSM8K & ARC-C \\
\midrule
Qwen3-8B & 0.4894 & 0.7061 & 0.7973 & 0.5660 \\
Qwen2-72B base & 0.3777 & 0.8376 & 0.8764 & 0.6348 \\
\bottomrule
\end{tabular}
\vspace{-8pt}
\end{table}

\subsection{Llama-3.2-3B native run: resolves the selection-transfer tie}
\label{app:llama32_3b}

{
\textbf{Setup and results.} Transferring the Qwen-derived selection to Llama-3.2-3B gives a MtS-Stratified gap of $-0.030$ ($p{=}0.45$, $n{=}5$). Recalibrating per-sample values on Llama while retaining the Qwen-derived mixture gives a gap of $+0.091$ ($p{=}0.0067$, paired, $n{=}5$). Table~\ref{tab:llama32_3b_phase5} reports means and variability; Table~\ref{tab:llama32_3b_per_domain} reports domain-level results for seed $42$.
\par}

\begin{table}[h]
\vspace{-5pt}
\centering
\caption{Utility on Llama-3.2-3B with LoRA $r{=}16$ and $m{=}1K$. Entries are means $\pm$ standard deviations over five seeds ($\dagger$: ten seeds). Utility differences and paired tests use Stratified as the reference.}
\label{tab:llama32_3b_phase5}
\small
\begin{tabular}{l c c c c}
\toprule
\textbf{Method} & \textbf{Mean} & \textbf{Std} & \textbf{$\Delta$ vs.\ Strat} & \textbf{$p$ (paired)} \\
\midrule
Mixture-then-Select (Qwen $\pi^*$ + Llama $v_{\text{agg}}$) & $\mathbf{+0.281}$ & $0.035$ & $\mathbf{+0.091}$ & $\mathbf{0.0067}$ \\
Stratified (Llama $v_{\text{agg}}$ within-domain)         & $0.190$           & $0.051$ & N/A                  & N/A \\
Random (Qwen-selection-transfer)                                  & $0.076$           & $0.056$ & $-0.114$             & $0.012$ \\
Facility Location (transfer)                                      & $-0.016$           & $0.051$ & $-0.206$             & $<0.001$ \\
DEITA (transfer)$^{\dagger}$                                      & $-0.109$           & $0.089$ & $-0.299$             & $<0.001$ \\
$k$-DPP (transfer)                                                & $-0.101$           & $0.038$ & $-0.291$             & $<0.001$ \\
Superfiltering (transfer)$^{\dagger}$                             & $-0.179$           & $0.063$ & $-0.369$             & $<0.001$ \\
Greedy-by-Value (transfer)                                        & $-0.203$           & $0.089$ & $-0.394$             & $<0.001$ \\
CRAIG (transfer)                                                  & $-0.204$           & $0.067$ & $-0.395$             & $<0.001$ \\
LESS (transfer)                                                   & $-0.238$           & $0.090$ & $-0.428$             & $<0.001$ \\
EL2N (transfer)                                                   & $-0.244$           & $0.055$ & $-0.434$             & $<0.001$ \\
\bottomrule
\end{tabular}
\par\smallskip
{\footnotesize\raggedright
MtS retains Qwen's mixture with Llama-recalibrated rankings; Stratified uses Llama rankings. Other rows transfer Qwen-selected subsets.\par}
\vspace{-8pt}
\end{table}

\begin{table}[h]
\vspace{-5pt}
\centering
\caption{Per-domain utility on Llama-3.2-3B for seed 42. Utility is $u_c=(L_c^{\mathrm{base}}-L_c^S)/L_c^{\mathrm{base}}$, where the two losses are measured on held-out domain-$c$ data before and after fine-tuning on subset $S$.}
\label{tab:llama32_3b_per_domain}
\footnotesize
\setlength{\tabcolsep}{4pt}
\begin{tabular}{l c c c c c c c c c}
\toprule
\textbf{Method} & math & code & gen.\ qa & dialog & reason. & nlp & cmplx & creative & \textbf{agg} \\
\midrule
Mixture-then-Select      & $0.325$ & $0.228$ & $0.284$ & $0.206$ & $0.295$ & $0.244$ & $0.218$ & $0.210$ & $\mathbf{+0.251}$ \\
Stratified            & $0.080$ & $0.173$ & $0.176$ & $0.126$ & $0.168$ & $0.177$ & $0.158$ & $0.158$ & $0.152$ \\
LESS                         & $-0.325$ & $-0.112$ & $-0.156$ & $-0.184$ & $-0.141$ & $-0.137$ & $-0.168$ & $-0.120$ & $-0.168$ \\
EL2N                         & $-0.306$ & $-0.188$ & $-0.297$ & $-0.254$ & $-0.207$ & $-0.252$ & $-0.191$ & $-0.187$ & $-0.235$ \\
\bottomrule
\end{tabular}
\vspace{-8pt}
\end{table}

\begin{figure}[h]
\centering
\includegraphics[width=0.85\linewidth]{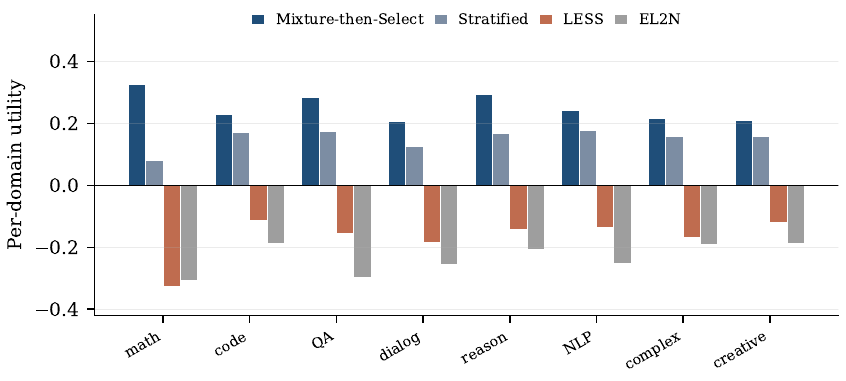}
\caption{Per-domain utility on Llama-3.2-3B at $N{=}50K$, $m{=}1K$, for seed 42. Numerical values are in \Cref{tab:llama32_3b_per_domain}.}
\label{fig:llama_per_domain}
\vspace{-5pt}
\end{figure}

\subsection{Capacity-gated domain mixture: data pool-$N$ and rank behavior}
\label{app:capacity-gating}
\label{app:pi_star_n200k}

{
The following results cover additional pool-size and adapter-rank settings. The headline calibration uses Qwen3-8B with LoRA rank $16$; the rank for each additional setting is reported below.
\par}

\begin{figure}[h]
\centering
\includegraphics[width=0.7\linewidth]{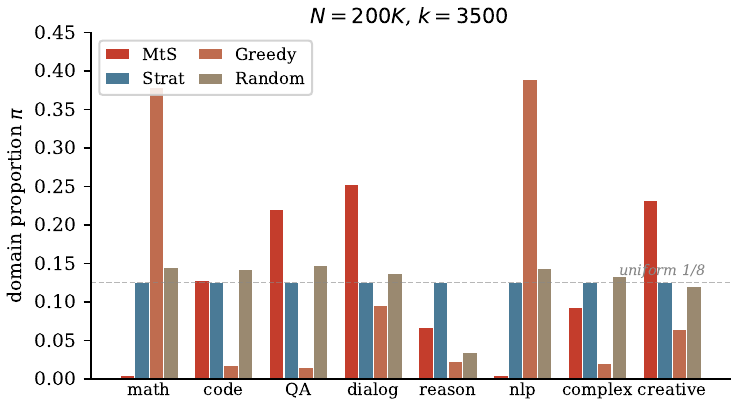}
\caption{Selected domain proportions on Qwen3-8B at $N{=}200K$, $m{=}3.5K$. The dashed line shows uniform allocation ($1/8$); the corresponding 50K comparison is in \Cref{fig:main_diagnostic}A.}
\label{fig:pi_star_n200k}
\vspace{-5pt}
\end{figure}

\textbf{Data Pool-$N$ gradient.}

{The table reports MtS and Stratified at three pool-size/rank settings. Because the ranks differ across rows, this comparison does not isolate pool size.\par}

\begin{table}[h]
\vspace{-5pt}
\centering
\footnotesize
\caption{MtS's utility difference from Stratified across pool sizes. The 50K condition uses LoRA $r{=}16$; 100K and 200K use $r{=}8$, so pool size and rank vary together.}
\label{tab:data pool-N-gradient}
\begin{tabular}{l c c c c}
\toprule
$N$ & LoRA $r$ & $\Delta$(MG $-$ Strat) & paired $p$ & $n$ \\
\midrule
50K  & 16 & $+$0.099 & $<$0.001 & 10 \\
100K & 8  & $+$0.036 & 0.062    & 5 \\
200K & 8  & $+$0.053 & 0.001    & 5 \\
\bottomrule
\end{tabular}
\vspace{-8pt}
\end{table}

{
These pool-size comparisons use ranks $8$ and $16$. Cross-model transfer is reported separately in Appendix~\ref{app:llama32_3b}.
\par}

\subsection{Bundle-level data pricing: setup, per-instance numbers, and negative results}
\label{app:pricing_pilot}

This section gives the full setup, per-instance results, and ablations behind \Cref{tab:pricing_pilot} in \Cref{sec:pricing_pilot}.

\begin{table}[htbp]
\vspace{-5pt}
\centering
\small

\caption{Bundle-pricing settings and comparison definitions.}
\label{tab:pricing_setup_compact}
\begin{tabular}{@{}p{0.24\linewidth}p{0.71\linewidth}@{}}
\toprule
\textbf{Setting} & \textbf{Specification} \\
\midrule
Pool & Math, code, general QA; five vendors per cluster, $50$ samples per indivisible bundle ($750$ samples total). \\
Prices / budgets & One price per cluster, $p_c\sim\mathcal{U}[0.5,1.5]$ independently; budgets $15\%,20\%,25\%$ of $\sum_c5p_c$; ten price/utility shuffles shared across methods. \\
Training & Llama-3.2-3B-Instruct, LoRA rank $16$, three epochs, learning rate $2\times10^{-4}$; five training seeds. \\
Candidate allocations & All feasible cluster counts $n_c\in\{0,\ldots,5\}$, selecting the top $n_c$ vendors by individual utility within each cluster. \\
Additive LP & Maximizes $\sum_c V_c(n_c)$, where $V_c$ is five-seed held-out utility from training on that cluster alone. \\
Joint-aware MtS & Scores the same feasible allocations by measured joint-training utility. \\
Oracle reference & All vendor combinations feasible at $25\%$ budget screened at seed $42$; top $50$ refined with four additional seeds. Best five-seed mean among refined candidates feasible at each budget. \\
\bottomrule
\end{tabular}
\vspace{-8pt}
\end{table}

{
\textbf{Per-instance results.} Table~\ref{tab:pricing_pilot_full} reports all $30$ cells. Mean utility is $0.240$ for joint-aware MtS and $0.160$ for additive LP. The reference oracle refines only the top $50$ single-seed candidates; MtS exceeds this restricted reference in two cells.
\par}

\begin{table}[!t]
\vspace{-5pt}
\centering
\footnotesize
\setlength{\tabcolsep}{4pt}
\caption{Per-instance data-pricing results on Llama-3.2-3B-Instruct with LoRA $r{=}16$. Entries are means $\pm$ standard deviations over five seeds, covering 10 instances at each of three budgets.}
\label{tab:pricing_pilot_full}
\resizebox{\linewidth}{!}{
\begin{tabular}{@{}cc ccc cc@{}}
\toprule
\textbf{inst} & \textbf{$B\%$} & \textbf{Add.\ LP} & \textbf{MtS} (joint-aware) & \textbf{Oracle} & \textbf{$\Delta$ MtS$-$Add} & \textbf{Recovery} \\
\midrule
$0$ & $15$ & $0.172$            & $\mathbf{0.201{\pm}0.014}$ & $0.229$         & $+0.029$ & $50.9\%$ \\
$0$ & $20$ & $0.139$            & $\mathbf{0.235{\pm}0.009}$ & $0.236$         & $+0.096$ & $99.7\%$ \\
$0$ & $25$ & $0.140$            & $\mathbf{0.240{\pm}0.022}$ & $0.240$         & $+0.100$ & $100\%$  \\
$1$ & $15$ & $0.180$            & $\mathbf{0.215{\pm}0.015}$ & $0.234$         & $+0.034$ & $64.0\%$ \\
$1$ & $20$ & $0.136$            & $\mathbf{0.215{\pm}0.015}$ & $0.234$         & $+0.079$ & $80.3\%$ \\
$1$ & $25$ & $0.136$            & $\mathbf{0.247{\pm}0.014}$ & $0.247$         & $+0.112$ & $100\%$  \\
$2$ & $15$ & $0.233$            & $\mathbf{0.255{\pm}0.019}$ & $0.246$         & $+0.022$ & $>100\%$\,\dag \\
$2$ & $20$ & $0.173$            & $\mathbf{0.255{\pm}0.021}$ & $0.263$         & $+0.082$ & $91.2\%$ \\
$2$ & $25$ & $0.159$            & $\mathbf{0.262{\pm}0.018}$ & $0.268$         & $+0.103$ & $94.2\%$ \\
$3$ & $15$ & $0.184$            & $\mathbf{0.240{\pm}0.014}$ & $0.247$         & $+0.056$ & $89.3\%$ \\
$3$ & $20$ & $0.141$            & $\mathbf{0.240{\pm}0.014}$ & $0.247$         & $+0.099$ & $93.6\%$ \\
$3$ & $25$ & $0.141$            & $\mathbf{0.245{\pm}0.014}$ & $0.247$         & $+0.103$ & $98.1\%$ \\
$4$ & $15$ & $0.178$            & $\mathbf{0.248{\pm}0.014}$ & $0.248$         & $+0.070$ & $100\%$  \\
$4$ & $20$ & $0.138$            & $\mathbf{0.248{\pm}0.014}$ & $0.248$         & $+0.109$ & $100\%$  \\
$4$ & $25$ & $0.089$            & $\mathbf{0.253{\pm}0.013}$ & $0.253$         & $+0.164$ & $100\%$  \\
$5$ & $15$ & $0.220$            & $\mathbf{0.253{\pm}0.017}$ & $0.253$         & $+0.034$ & $>100\%$\,\dag \\
$5$ & $20$ & $0.185$            & $\mathbf{0.258{\pm}0.021}$ & $0.265$         & $+0.073$ & $91.1\%$ \\
$5$ & $25$ & $0.185$            & $\mathbf{0.258{\pm}0.021}$ & $0.265$         & $+0.073$ & $91.1\%$ \\
$6$ & $15$ & $0.162$            & $\mathbf{0.190{\pm}0.019}$ & $0.190$         & $+0.028$ & $100\%$  \\
$6$ & $20$ & $0.101$            & $\mathbf{0.249{\pm}0.016}$ & $0.249$         & $+0.149$ & $100\%$  \\
$6$ & $25$ & $0.116$            & $\mathbf{0.254{\pm}0.015}$ & $0.254$         & $+0.138$ & $100\%$  \\
$7$ & $15$ & $0.233$            & $\mathbf{0.248{\pm}0.010}$ & $0.248$         & $+0.015$ & $100\%$  \\
$7$ & $20$ & $0.157$            & $\mathbf{0.248{\pm}0.010}$ & $0.248$         & $+0.091$ & $100\%$  \\
$7$ & $25$ & $0.157$            & $\mathbf{0.248{\pm}0.010}$ & $0.248$         & $+0.091$ & $100\%$  \\
$8$ & $15$ & $0.196$            & $\mathbf{0.213{\pm}0.011}$ & $0.229$         & $+0.017$ & $52.5\%$ \\
$8$ & $20$ & $0.140$            & $\mathbf{0.244{\pm}0.010}$ & $0.244$         & $+0.104$ & $99.9\%$ \\
$8$ & $25$ & $0.124$            & $\mathbf{0.251{\pm}0.017}$ & $0.251$         & $+0.127$ & $100\%$  \\
$9$ & $15$ & $0.209$            & $\mathbf{0.230{\pm}0.006}$ & $0.230$         & $+0.021$ & $100\%$  \\
$9$ & $20$ & $0.153$            & $\mathbf{0.230{\pm}0.006}$ & $0.230$         & $+0.077$ & $100\%$  \\
$9$ & $25$ & $0.126$            & $\mathbf{0.230{\pm}0.011}$ & $0.230$         & $+0.105$ & $100\%$  \\
\midrule
\multicolumn{2}{l}{\emph{Mean across $30$ cells}} & $0.160$ & $\mathbf{0.240}$ & $0.244$ & $+0.080$ & $93.2\%$ \\
\bottomrule
\end{tabular}
}\\[2pt]
\dag\,MtS exceeds Stage-B oracle in these cells (Stage~B's top-$50$ truncation excluded the bundle MtS happened to pick; its $5$-seed mean turned out higher).
\vspace{-8pt}
\end{table}

{
\textbf{Within-cluster scoring controls.} We tested replacing each solo-utility curve by its cumulative maximum and capping each cluster's allocation at its solo-utility maximizer. Neither changed additive LP's selection in any of the $30$ cells. Increasing budget does not consistently increase its measured joint-training utility.
\par}

\section{Broader impact}
\label{app:impact}

\textbf{Positive impacts.}
Interaction-aware data selection can reduce waste in LLM fine-tuning by helping practitioners assemble smaller, higher-value training subsets.
The framework may be especially useful for resource-constrained groups that cannot afford exhaustive trial-and-error over large data pools.

\textbf{Negative impacts.}
Interaction-aware selection could amplify existing dataset biases if certain domains or styles are repeatedly judged redundant and filtered out.
If redundancy and complementarity correlate with demographic or linguistic attributes, optimizing only for measured utility may under-select already underrepresented data.

\textbf{Mitigations.}
The matroid and multi-objective extensions can encode diversity or coverage constraints directly.
More broadly, the framework makes interactions explicit, which can help practitioners audit whether seemingly efficient selections are systematically excluding important data types.

\end{document}